%% file: Main_arxiv.tex
\documentclass[onefignum,onetabnum,sort&compress]{siamonline190516}

\input{ex_shared}

\begin{document}
%
%

%
\maketitle              
\begin{abstract}
In this paper, we propose a spectral method for deriving functions that are jointly smooth on multiple observed manifolds.
This allows us to register measurements of the same phenomenon by heterogeneous sensors, and to reject sensor-specific noise.
Our method is unsupervised and primarily consists of two steps. First, using kernels, we obtain a subspace spanning smooth functions on each separate manifold. Then, we apply a spectral method to the obtained subspaces and discover functions that are jointly smooth on all manifolds.
We show analytically that our method is guaranteed to provide a set of orthogonal functions that are as jointly smooth as possible, ordered by increasing Dirichlet energy from the smoothest to the least smooth.
In addition, we show that the extracted functions
can be efficiently extended to unseen data using the Nystr\"{o}m method.
We demonstrate the proposed method on both simulated and real measured data and compare the results to nonlinear variants of the seminal Canonical Correlation Analysis (CCA).
Particularly, we show superior results for sleep stage identification. In addition, we show how the proposed method can be leveraged for finding minimal realizations of parameter spaces of nonlinear dynamical systems. 
\end{abstract}




%
%
\section{Introduction}
Modern data acquisition typically involves the use of multiple modalities.
Indeed, nowadays, many high- and low-end devices are equipped with multiple sensors, often of different types, giving rise to multimodal data collections. While such acquisitions have become popular only in the last two decades, the discovery of informative features from multimodal observations has been a central problem in data analysis for many years.

Perhaps the first attempt, and the most well-known and widely used algorithm for building efficient representations for multimodal data sets is the celebrated Canonical Correlation Analysis (CCA) \cite{harold1936relations}. Broadly, CCA builds linear projections of the data sets such that the correlation between them is maximized. Despite the extensive body of evidence proving CCA as extremely useful, it suffers from a few prominent shortcomings, facilitating a large body of work improving CCA. 
Kernel CCA (KCCA) \cite{lai2000kernel,akaho2006kernel} extends the linear setting and considers instead nonlinear projections by using the kernel trick \cite{mercer-1909}. 
Multi-view and weighted versions of KCCA were proposed in \cite{yu2013kernel,lindenbaum-2020}.
Deep learning tools allowed the significant extension of the space of possible projections, giving rise to Deep CCA \cite{andrew2013deep}.
The Non-parametric CCA (NCCA) \cite{michaeli2016nonparametric} provides a closed-form solution to a nonlinear variant of the CCA optimization problem, assuming that the projections belong to some Reproducing Kernel Hilbert Space (RKHS).

In addition to the statistical viewpoint adopted by CCA and its variants, recently, geometric and manifold learning methods were proposed as well \cite{eynard2015multimodal,lederman2018learning,lindenbaum-2020} for the purpose of finding common representations to multimodal data. Particularly, in \cite{eynard2015multimodal}, the data are viewed as high-dimensional points residing on low-dimensional manifolds, and the objective is to obtain a basis that simultaneously diagonalizes the Laplacians of all the manifolds. Since this objective cannot be fully achieved, the authors introduce certain off-diagonal penalties to their objective function.

Relying on these recent developments, in this paper, we propose a different approach based on manifold learning.
We retain the setting as in \cite{eynard2015multimodal,Yair-2017b,lederman2018learning,lindenbaum-2020}, assuming that the multimodal data lie on multiple manifolds, building the Laplacian of each manifold and then applying spectral analysis. Instead of products of kernel matrices, however, our approach focuses on the space of real functions defined on these manifolds. We construct a joint function space for all observations, sorted by decreasing smoothness.
Such an approach has roots in operator theoretic dynamical system analysis, dating back to Koopman \cite{koopman-1931}, and has recently regained a lot of attention in learning dynamical systems from observations \cite{budisic-2012,williams-2015,dietrich-2016b} as well as in computer graphics \cite{ovsjanikov2012functional}.
Concretely, we define smooth functions on a manifold as functions that can be spanned by the top eigenfunctions of the Laplace-Beltrami operator of the manifold.
Then, our main goal is to discover functions that are jointly smooth on all observed manifolds.
We show that such jointly smooth functions are of great interest since they represent and parametrize the commonality between the observed manifolds, which provides useful information on the co-relationship between the respective data sets obtained through possibly different modalities. 
To accomplish our goal, we propose a two-step algorithm.
First, for each manifold, we extract a subspace spanning the family of smooth functions on that manifold. This is achieved by constructing a kernel on the data approximating the Laplacian.
Second, we apply singular value decomposition (SVD) to the union of all subspaces and show analytically that the obtained singular vectors indeed represent jointly smooth functions on the observed manifolds.
We test the proposed algorithm on real measured multimodal data. Specifically, we demonstrate the ability to accurately identify the sleep stage from various simultaneous physiological recordings. In addition, we show how joint latent spaces can help in positional alignment from multiple video feeds of a race track, and how this method can be exploited for finding effective parameters of nonlinear dynamical systems from observations in a model-free manner.

Our main contributions are as follows. First, we introduce an approach for multimodal data analysis based on the notion of jointly smooth functions on manifolds.
Second, we propose an algorithm for finding such jointly smooth functions with theoretical guarantees. We show that the proposed algorithm supports multiple (more than two) manifolds, and that it can be efficiently computed and extended to unseen data.
We also discuss how the jointly smooth functions can be used to embed the data common to the different sensors in a common latent space.
Third, we show how jointly smooth functions can be constructed efficiently, with an almost linear time and space complexity.
Fourth, we showcase the proposed algorithm on simulated and real measured data sets. Specifically, we consider a sleep stage identification task and observe superior results compared to competing methods. We also demonstrate how to obtain a common latent space for two video streams, and identify effective parameters in a dynamical systems context.

\section{Smooth Functions}
    \label{sec:pre}
\subsection{Smooth Functions on Manifolds}
    \label{sub:Continuous}
    Let $\mathcal{M}_x\subseteq\mathbb{R}^{d}$ be a smooth, compact manifold embedded in a $d$-dimensional Euclidean space.
    A function $f:\mathcal{M}_{x}\to\mathbb{R}$ is said to be smooth if $f\in C^{\infty}(\mathcal{M}_x,\mathbb{R})$, that is, the function $f$ and all of its derivatives are continuous.
    In this section, we define a different notion of smoothness related to the Laplace-Beltrami operator.
    The Laplace-Beltrami operator $L_x$ is a linear operator generalizing the Laplacian on Euclidean spaces to Riemannian manifolds~\cite{rosenberg-1997}.
    The eigenfunctions $\psi_i$ of the Laplace-Beltrami operator span a dense subset of the function space $H^0=L^2(\mathcal{M}_x,\mathbb{R})$. The eigenvalues of the Laplace-Beltrami are real (and non-negative), so we can sort the associated eigenfunctions $\psi_i$ such that $\lambda_i \leq \lambda_j$ for $i<j$.
    We say that $\psi_i$ is smoother than $\psi_j$ if $\lambda_i < \lambda_j$. For example, the constant function $\psi \equiv 1$ is an eigenfunction with eigenvalue $\lambda = 0$. The higher the value of $\lambda_i$, the more oscillatory is the corresponding eigenfunction (that is, less smooth). For more details, we refer to~\cite{giannakis-2017}.
    It was shown that the best representation basis, in terms of truncated representation of functions $f:\mathcal{M}_{x}\to\mathbb{R}$ such that $\Vert \nabla f\Vert_2 \leq1$, are in fact the eigenfunctions of the Laplace-Beltrami operator $L_x$ \cite{aflalo2015optimality,aflalo2016best,brezis2017rigidity}.
    Thus, in that sense, we say that $f_{i}:\mathcal{M}_{x}\to\mathbb{R}$ is smoother than $f_{j}:\mathcal{M}_{x}\to\mathbb{R}$ if $\Vert L_{x}f_{i}\Vert_2/\Vert f_{i}\Vert_2 <\Vert L_{x}f_{j}\Vert_2/\Vert f_{j}\Vert_2$. Note that we write $\Vert\cdot\Vert_2$ for the $L^2$ norm on $H^0$, but we use the same notation if we consider finite-dimensional approximations and the corresponding 2-norm in $N$ dimensions later.

\subsection{Smooth Functions on Data}
    Consider a set of points $\{ \boldsymbol{x}_{i}\in\mathcal{M}_{x} \} _{i=1}^{N}$ residing on a low-dimensional manifold $\mathcal{M}_{x}\subseteq\mathbb{R}^{d}$ embedded in a $d$-dimensional Euclidean space.
    Many unsupervised data analysis methods use kernels to effectively represent data \cite{scholkopf1997kernel,belkin2002laplacian,shawe2004kernel,coifman2006diffusion}.
    Kernels are typically symmetric and positive functions~\cite{mercer-1909}, represented as matrices by evaluating them on pairs of data points. Perhaps the most widely used kernel is the Gaussian kernel, given by
    \[
    \boldsymbol{K}_{x}\left[i,j\right]=\exp\left(-\frac{ \smash[b]{\Vert \boldsymbol{x}_{i}-\boldsymbol{x}_{j}\Vert _{2}^{2}} }{ \smash[b]{2\sigma_x^{2}} } \right),\qquad\boldsymbol{K}_x\in\mathbb{R}^{N\times N}.
    \]
    This Gaussian kernel $\boldsymbol{K}_x$ approximates the operator $\exp(-\sigma_x L_{x})$, whose eigenfunctions are the same as the eigenfunctions of $L_x$~\cite{yosida-1995}. 
    Similar to the matrix representation of the kernel, we use a vector representation of functions through their evaluation on the data.
    Let $\boldsymbol{w}_{i}$ be a unit norm eigenvector of $\boldsymbol{K}_{x}$ such that $\boldsymbol{K}_{x}\boldsymbol{w}_{i}=\lambda_{i}\boldsymbol{w}_{i}$, where $\lambda_{1}\geq\lambda_{2}\geq\dots\geq\lambda_{N}$.
    These eigenvectors form a basis ordered from the smoothest vector
    $\boldsymbol{w}_{1}$ to the most oscillatory vector $\boldsymbol{w}_{N}$, where smoothness is defined in a manner similar to the smoothness with respect to the Laplace-Beltrami operator $L_x$. The following definition is inspired by the Dirichlet energy $\Vert\nabla {f}\Vert^2_2$ of a function~\cite{giannakis-2017}, and adapted to kernel matrices. Note that Dirichlet energy decreases for smoother functions, while our smoothness score increases.

\begin{definition}{\textbf{Truncated smoothness score.}}\label{def:smoothness score}
For a number $d<N$ and a given function $\boldsymbol{f}\in\mathbb{R}^N$, define its $d$-truncated smoothness score $E_x^d({\boldsymbol{f}})$ with respect to a kernel $\boldsymbol{K}_{x}$ by
\begin{equation}
    E_x^{d}({\boldsymbol{f}})\coloneqq\sum_{i=1}^d\left\langle \boldsymbol{w}_i, \boldsymbol{f}\right\rangle^2 = \big\Vert \boldsymbol{W}_{x}^{T}\boldsymbol{f}\big\Vert _{2}^{2},
\end{equation}
where $ \smash{ \boldsymbol{W}_{x}\coloneqq [\begin{matrix}\boldsymbol{w}_{1} & \boldsymbol{w}_{2} & \cdots & \boldsymbol{w}_{d}\end{matrix}] \in\mathbb{R}^{N\times d}} $ is a concatenation of sorted eigenvectors of $\boldsymbol{K}_x$.

\end{definition}

\begin{definition}{\textbf{$d$-smooth function.}}
    \label{def:smooth}
A function $\boldsymbol{f}\in\mathbb{R}^N$ is called ``$d$-smooth on $\mathcal{M}_x$'' if its $d$-truncated smoothness score with respect to $\boldsymbol{K}_{x}$ is equal to its squared $2$-norm, i.e.  $E_x^{d}({\boldsymbol{f}})=\left\Vert \boldsymbol{f}\right\Vert^2_{2}$.
\end{definition}

\begin{remark}
For any $\boldsymbol{f}\in\mathbb{R}^N$, we have $$E_x^d({\boldsymbol{f}})=\left\Vert \boldsymbol{f}\right\Vert^2_{2}\iff \boldsymbol{f}\in \text{span}\left(\boldsymbol{W}_{x}\right).$$
\end{remark}

\begin{definition}{\textbf{Jointly smooth function.}}\label{def:joint}
A function $\boldsymbol{f}\in\mathbb{R}^N$ is called ``jointly
smooth on $\mathcal{M}_x$ and $\mathcal{M}_y$'' if it is $d$-smooth with respect to both $\boldsymbol{K}_{x}$ and $\boldsymbol{K}_{y}$, i.e.  $E_x^{d}({\boldsymbol{f}}) = E_y^{d}({\boldsymbol{f}}) = \left\Vert \boldsymbol{f}\right\Vert _{2}^2$.
\end{definition}

\begin{remark}
Empirical functions, especially in the presence of observation noise, typically have a nonzero inner product with all eigenvectors of the kernel $\boldsymbol{K}_x$. Even though they would then not be smooth according to Definition \ref{def:smooth}, we can still compare them using our smoothness score \eqref{def:smoothness score} and say that $\boldsymbol{f}_{i}\in\mathbb{R}^{N}$ is smoother
than $\boldsymbol{f}_{j}\in\mathbb{R}^{N}$ if $\Vert \boldsymbol{W}_{x}^{T}\boldsymbol{f}_{i}\Vert _{2}/\Vert \boldsymbol{f}_{i}\Vert _{2}>\Vert \boldsymbol{W}_{x}^{T}\boldsymbol{f}_{j}\Vert _{2}/\Vert \boldsymbol{f}_{j}\Vert _{2}$.
\end{remark}

\section{Problem Formulation}
    \label{sec:ProblemFormulation}
Let $\mathcal{M}_{x}\subseteq\mathbb{R}^{d_{x}}$ and $\mathcal{M}_{y}\subseteq\mathbb{R}^{d_{y}}$ be two manifolds embedded in high-dimensional ambient spaces.
We later extend this formulation to more than two manifolds. Consider two sets of observations $\left\{ \boldsymbol{x}_{i}\in\mathcal{M}_{x}\right\} _{i=1}^{N}$ and $\left\{ \boldsymbol{y}_{i}\in\mathcal{M}_{y}\right\} _{i=1}^{N}$ such that $\left(\boldsymbol{x}_{i},\boldsymbol{y}_{i}\right)$ is a corresponding pair.
In this work, we search for an orthogonal set of functions $\boldsymbol{f}_m \in \mathbb{R}^N$ such that each $\boldsymbol{f}_m$ is jointly smooth on $\mathcal{M}_{x}$ and $\mathcal{M}_{y}$. 
Note that in this discrete setting, the functions $\boldsymbol{f}_{m}\in\mathbb{R}^{N}$ are defined both on $\left\{ \boldsymbol{x}_{i}\in\mathcal{M}_{x}\right\} _{i=1}^{N}$ and $\left\{ \boldsymbol{y}_{i}\in\mathcal{M}_{y}\right\} _{i=1}^{N}$, but in fact, they can be viewed as sampled versions of continuous (and smooth) functions $f_{m}^{x}:\mathbb{\mathcal{M}}_{x}\to\mathbb{R}$ and ${f}_{m}^{y}:\mathbb{\mathcal{M}}_{y}\to\mathbb{R}$ such that $f_{m}^{x}\left(\boldsymbol{x}_{i}\right)=f_{m}^{y}\left(\boldsymbol{y}_{i}\right)=\boldsymbol{f}_{m}\left[i\right]$.
We posit that such smooth functions are of great interest since they represent and parametrize the common part between the two manifolds.
Therefore, the remainder of this paper revolves around the question: how can we find functions that are jointly smooth, given only the kernels $\boldsymbol{K}_{x}$ and $\boldsymbol{K}_{y}$ associated with the observed data?

Figure \ref{fig:Toy1} provides an illustrative example.
Consider pairs of observations $\left\{ \left(\boldsymbol{x}_{i},\boldsymbol{y}_{i}\right)\right\} _{i=1}^{N}$ such that $\boldsymbol{x}_{i}\in\mathcal{M}_{x}\subset\mathbb{R}^{2}$ resides on a spiral in $\mathbb{R}^2$ and $\boldsymbol{y}_{i}\in\mathcal{M}_{y}\subset\mathbb{R}^{3}$ resides on a torus in $\mathbb{R}^3$ as depicted in Fig. \ref{fig:Toy1}(a).
Figures \ref{fig:Toy1}(b,c) present smooth functions on $\mathcal{M}_x$ and $\mathcal{M}_y$ that are not jointly smooth. Conversely, Fig. \ref{fig:Toy1}(d) presents a jointly smooth function.

\begin{figure}[h]
    \centering
    \includegraphics[width=1\columnwidth]{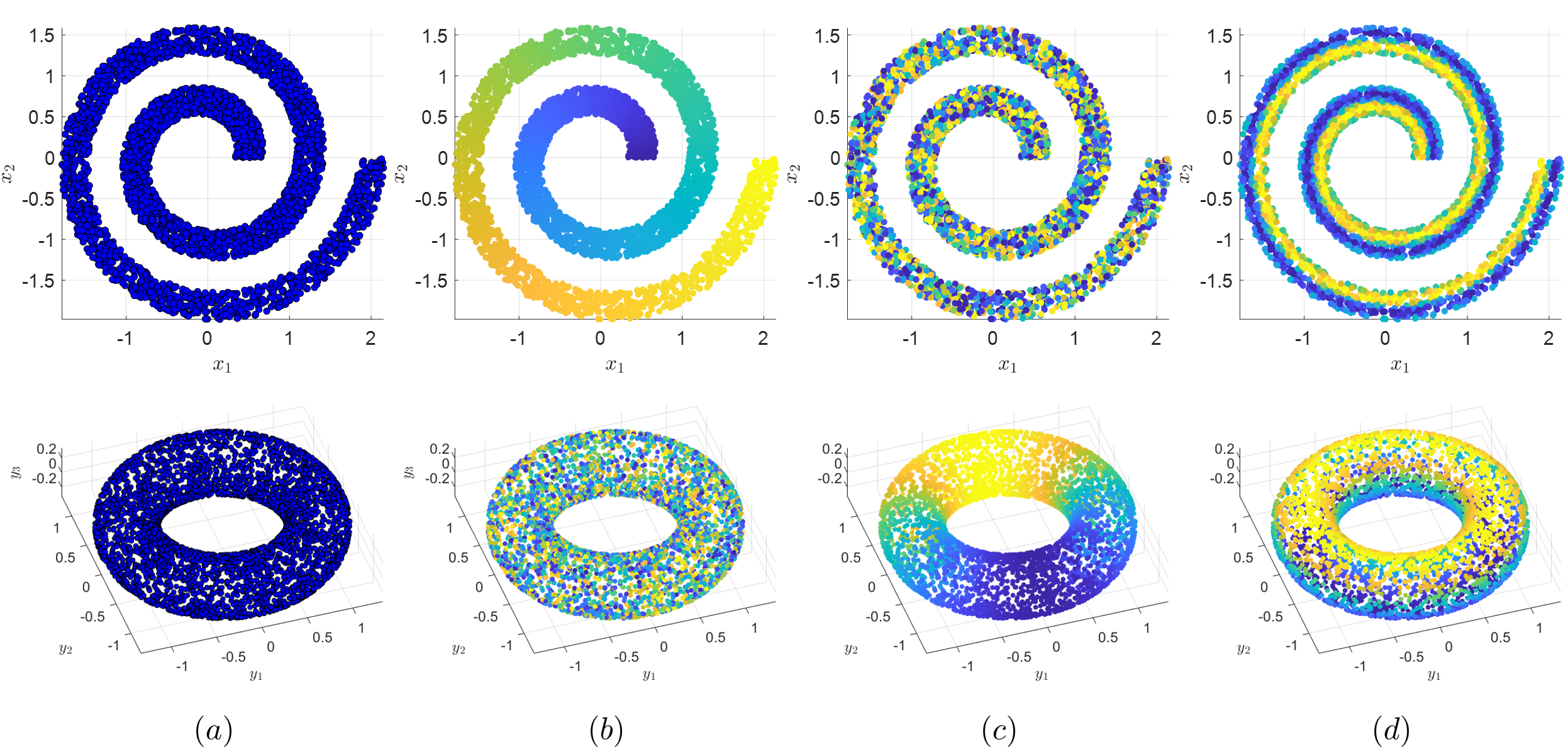}
    \caption{(a) Points $\boldsymbol{x}_{i}\in\mathcal{M}_{x}$ on a spiral at the top row and points $\boldsymbol{y}_{i}\in\mathcal{M}_{y}$ on a torus at the bottom row, where each $\left(\boldsymbol{x}_{i},\boldsymbol{y}_{i}\right)$ is a corresponding pair.
    (b) The two manifolds colored according to the \textit{same} function $\boldsymbol{f}_x$, which is smooth on $\mathcal{M}_{x}$.
    (c) The two manifolds colored according to the \textit{same} function $\boldsymbol{f}_y$, which is smooth on $\mathcal{M}_{y}$.
    We observe that $\boldsymbol{f}_{x}$ is smooth on $\mathcal{M}_{x}$ but it is not smooth on $\mathcal{M}_{y}$. Similarly, $\boldsymbol{f}_{y}$ is smooth on $\mathcal{M}_{y}$ but it is not smooth on $\mathcal{M}_{x}$.
    (d) The two manifolds colored according to the \textit{same} jointly smooth function $\boldsymbol{f}_1$ on $\mathcal{M}_{x}$ and $\mathcal{M}_{y}$.
    For more details, see Section \ref{sub:toy}.}
    

    \label{fig:Toy1}
\end{figure}

\section{Proposed Method}
    \label{sec:sol}
The proposed method is based on the following lemma.
\begin{lemma}
    \label{lem:equal}
Consider $\boldsymbol{A},\boldsymbol{B}\in\mathbb{R}^{N\times d}$
such that $\boldsymbol{A}^{T}\boldsymbol{A}=\boldsymbol{B}^{T}\boldsymbol{B}=\boldsymbol{I}_{d}$.
Let $\boldsymbol{W}\coloneqq\left[\begin{matrix}\boldsymbol{A} & \boldsymbol{B}\end{matrix}\right]\in\mathbb{R}^{N\times2d}$.
Then, the following decomposition of $\boldsymbol{W}$ 
\[
\boldsymbol{W}=\boldsymbol{U}\boldsymbol{\Sigma}\boldsymbol{V}^{T}
\]
is an SVD, such that
\[
\left\Vert \boldsymbol{A}^{T}\boldsymbol{u}_{i}\right\Vert _{2}^{2}=\left\Vert \boldsymbol{B}^{T}\boldsymbol{u}_{i}\right\Vert _{2}^{2}=\frac{1}{2}\sigma_{i}^{2}\qquad\forall i,
\]
where $\boldsymbol{u}_{i}$ is the $i$th column of $\boldsymbol{U}$
and $\sigma_{i}$ equals $\boldsymbol{\Sigma}\left[i,i\right]$ if
$i\leq2d$ and $0$ otherwise.
In addition, $\boldsymbol{V}=\frac{1}{\sqrt{2}}\left[\begin{matrix}\boldsymbol{Q} & \boldsymbol{Q}\\
\boldsymbol{R} & -\boldsymbol{R}
\end{matrix}\right]\in\mathbb{R}^{2d\times2d}$, $\boldsymbol{\Sigma}^{2}=\left[\begin{matrix}\boldsymbol{I}+\boldsymbol{\Gamma} & \boldsymbol{0}\\
\boldsymbol{0} & \boldsymbol{I}-\boldsymbol{\Gamma}
\end{matrix}\right]\in\mathbb{R}^{2d\times2d}$, and $\boldsymbol{U}=\boldsymbol{W}\boldsymbol{V}\boldsymbol{\Sigma}^{\dagger}\in\mathbb{R}^{N\times2d}$,
where $\boldsymbol{\Sigma}^{\dagger}$ is the pseudo-inverse of $\boldsymbol{\Sigma}$ and $\boldsymbol{A}^{T}\boldsymbol{B}=\boldsymbol{Q}\boldsymbol{\Gamma}\boldsymbol{R}^{T}\in \mathbb{R}^{d\times d}$ is an SVD.
\end{lemma}
The proof appears in the Supplementary Material (SM).

The consequence of Lemma \ref{lem:equal} in our setting is Corollary \ref{col:col}, which is derived by replacing the matrices $\boldsymbol{A}$ and $\boldsymbol{B}$ in Lemma \ref{lem:equal} with the matrices $\boldsymbol{W}_{x}$ and $\boldsymbol{W}_{y}$, consisting of the first $d$ dominant (normalized) eigenvectors of kernels $\boldsymbol{K}_{x}$ and $\boldsymbol{K}_{y}$, respectively.
\begin{corollary}
    \label{col:col}
    The most dominant left singular vectors\textbf{ $\boldsymbol{U}\in\mathbb{R}^{N\times d}$
    }of $\boldsymbol{W}\coloneqq\left[\begin{matrix}\boldsymbol{W}_{x} & \boldsymbol{W}_{y}\end{matrix}\right]$
    satisfy
$$\left\Vert \boldsymbol{W}_{x}^{T}\boldsymbol{u}_{i}\right\Vert _{2}^{2}=\left\Vert \boldsymbol{W}_{y}^{T}\boldsymbol{u}_{i}\right\Vert _{2}^{2}=\frac{1}{2}\sigma_{i}^{2}=\frac{1}{2}\left(1+\gamma_{i}\right),\qquad i\leq d,$$
where $\gamma_{i}=\boldsymbol{\Gamma}\left[i,i\right]$. 
\end{corollary}
By definition, $\gamma_i$ are the singular values of $\boldsymbol{W}_x^T \boldsymbol{W}_y$, and thus, are in fact the cosine of the principal angles $\theta_i$ between the two sub-spaces spanned by $\boldsymbol{W}_{x}$ and $\boldsymbol{W}_{y}$, that is $\gamma_{i}=\cos\left(\theta_{i}\right)$, where $0\leq\gamma_{i}\leq1$. Therefore, for each singular value $\gamma_{i}=1$, there exists a common direction between the two sub-spaces. In addition, directions that are not common are associated with $\gamma_{i}<1$, and the closer $\gamma_i$ is to the value $1$, the smaller is the angle between the two respective directions.

By corollary \ref{col:col}, only when there exists a common direction, the respective singular value $\gamma_i = 1$, $\sigma_i^2=2$, and $\boldsymbol{u}_i$ strictly satisfies Definition \ref{def:joint} of jointly smooth functions. This typically holds for the constant vector $\boldsymbol{u}_1$, satisfying  $\Vert \boldsymbol{W}_{x}^{T}\boldsymbol{u}_{1}\Vert _{2}^{2}=\Vert \boldsymbol{W}_{y}^{T}\boldsymbol{u}_{1}\Vert _{2}^{2}=1$ independent of the particular kernels $\boldsymbol{K}_{x}$ and $\boldsymbol{K}_{y}$. Other non-degenerate singular vectors representing common directions may exist, yet, due to noise and other possible distortions, it is likely that the obtained singular vectors $\boldsymbol{u}_i$ will not be strictly smooth, that is $\Vert \boldsymbol{W}_{x}^{T}\boldsymbol{u}_{i}\Vert _{2}^{2}<1$. Still, by Corollary \ref{col:col}, the most dominant left singular vectors of $\boldsymbol{U}$ are ordered from the ``most'' jointly smooth function to the least with respect to their $d$-truncated smoothness score.
This implies that taking the top $M$ columns $\boldsymbol{u}_i \in\mathbb{R}^N$ of $\boldsymbol{U}$ accomplishes the goal of finding jointly smooth functions. The entire procedure is summarized in Algorithm \ref{alg:2views}. Overall, the output of Algorithm \ref{alg:2views} are smooth functions $\boldsymbol{f}_i$, which are the top $M$ left singular vectors $\boldsymbol{u}_i$.

\begin{remark}
When the compact SVD of $\boldsymbol{W}\coloneqq\left[\begin{matrix}\boldsymbol{W}_{x} & \boldsymbol{W}_{y}\end{matrix}\right]$ is unique, any method of computing the SVD can be applied. However, when the compact SVD is not unique, one needs to construct the SVD decomposition as prescribed by the second part of Lemma \ref{lem:equal} and implemented in Steps 3 and 4 of Algorithm \ref{alg:2views}. Importantly, computing this specific SVD rather than applying a generic SVD algorithm is computationally efficient since it requires to decompose $\boldsymbol{W}_{x}^{T}\boldsymbol{W}_{y}\in\mathbb{R}^{d\times d}$ instead of $\boldsymbol{W}\in\mathbb{R}^{N\times2d}$, where $d < N$.
\end{remark}

\subsection{Choosing the value of $M$}
    \label{sub:ChooseM}
The discussion above implies that the obtained functions $\boldsymbol{u}_i$ might not be strictly smooth, that is $\Vert \boldsymbol{W}_{x}^{T}\boldsymbol{u}_{i}\Vert _{2}^{2}<1$. Hence, in Algorithm \ref{alg:2views} we use the top $M$ functions. The number of functions $M$ is therefore a hyperparameter of the algorithm that needs to be set a priori. We propose to set $M$ as the number of functions $\boldsymbol{u}_i$ satisfying $\Vert \boldsymbol{W}_{x}^{T}\boldsymbol{u}_{i}\Vert _{2}^{2}>E_0$, where $E_0$ is a suitable threshold.
Similarly to the jackstraw method \cite{chung2015statistical}, one can permute one of the observations, e.g. $\{ (\boldsymbol{x}_{i},\boldsymbol{y}_{\pi\left(i\right)}) \} $ where $\pi$ is a random permutation, and then, compute the matrix $\tilde{\boldsymbol{\Gamma}}$ from Lemma \ref{lem:equal}.
Except the first trivial singular value $\tilde{\gamma}_1 = 1$, the rest of the singular values $\tilde{\gamma}_i$ and the singular vectors $\tilde{\boldsymbol{u}}_i$ are associated with pure noise observations. Thus, the second singular value $\tilde{\gamma}_2$ provides us a threshold to the original singular values $\gamma_i$, that is, $\gamma_{i}>\tilde{\gamma}_{2}$, which leads to the following threshold
\[E_0=\frac{1}{2}\left(1+\tilde{\gamma}_{2}\right).\]
Alternatively, if we assume random (independent) observations we can provide an estimation for $\tilde{\gamma}_2$ since we can write $\tilde{\gamma}_{2}=\cos\left(\theta_{2}\right)$ where $\theta_2$ is the smallest principal angle between two random sub-spaces uniformly distributed (again, ignoring the first trivial constant sub-space). Based on \cite{johnstone2008multivariate}, we derived our proposed threshold
$$E_0=\frac{1}{2}+\frac{\sqrt{d-\frac{1}{2}}\sqrt{N-d-\frac{1}{2}}}{N-1}.$$
See the SM for the derivation of this threshold.

\subsection{Multiple Manifolds (Multiple Views)}
The proposed method can be extended to support multiple manifolds in a straight-forward manner. While Lemma \ref{lem:equal} is not valid for more than two manifolds, numerical tests consistently show that the na\"{i}ve extension provides satisfactory results. 
Consider $K>2$ sets of observations $\big\{ \boldsymbol{x}_{i}^{(1)}\in\mathcal{M}_{1}\big\} _{i=1}^{N},\big\{ \boldsymbol{x}_{i}^{(2)}\in\mathcal{M}_{2}\big\} _{i=1}^{N},\ldots,\big\{ \boldsymbol{x}_{i}^{(K)}\in\mathcal{M}_{K}\big\} _{i=1}^{N}$, where each $\big(\boldsymbol{x}_{i}^{(1)},\boldsymbol{x}_{i}^{(2)},\dots,\boldsymbol{x}_{i}^{(K)}\big)$ is a tuple of aligned observations. For $k=1,\ldots,K$, let $\boldsymbol{W}_{k}\in\mathbb{R}^{N\times d}$ be the $d$ dominant eigenvectors of the kernel $\boldsymbol{K}_{k}$ constructed from the observations of the $k$-th manifold. 
Denote $\boldsymbol{W}\coloneqq\left[\begin{matrix}\boldsymbol{W}_{1} & \boldsymbol{W}_{2} & \cdots & \boldsymbol{W}_{K}\end{matrix}\right]\in \mathbb{R}^{N\times Kd}$. 
Similarly to Corollary \ref{col:col}, we propose to compute the top $M$ left singular vectors of $\boldsymbol{W}$ and view them as the jointly smooth functions on all $K$ manifolds.
The entire extension is outlined in Algorithm \ref{alg:multi}. See Section \ref{sub:sleep} for supporting numerical results.


\begin{algorithm}[t]
\textbf{\uline{Input}}\textbf{:} $2$ observations
$\left\{ \boldsymbol{x}_{i},\boldsymbol{y}_{i}\right\} _{i=1}^{N}$
where $\boldsymbol{x}_{i}\in\mathbb{R}^{d_{x}}$ and $\boldsymbol{y}_{i}\in\mathbb{R}^{d_{y}}$.

\textbf{\uline{Output}}\textbf{:} $M$ joinly smooth functions $\left\{ \boldsymbol{f}_{m}\in\mathbb{R}^{N}\right\} _{m=1}^{M}$.
\begin{enumerate}
\item Compute the kernels $\boldsymbol{K}_{x},\boldsymbol{K}_{y}\in\mathbb{R}^{N\times N}$:
\[
\boldsymbol{K}_{x}\left[i,j\right]=\exp\left(-\frac{\Vert \boldsymbol{x}_{i}-\boldsymbol{x}_{j}\Vert _{2}^{2}}{2\sigma_{x}^{2}}\right), \qquad \boldsymbol{K}_{y}\left[i,j\right]=\exp\left(-\frac{ \smash[b]{ \Vert \boldsymbol{y}_{i}-\boldsymbol{y}_{j}\Vert _{2}^{2}}}{ \smash[b]{2\sigma_{y}^{2}}} \right)
\]
\item Compute $\boldsymbol{W}_{x},\boldsymbol{W}_{y}\in\mathbb{R}^{N\times d}$,
the first $d$ eigenvectors of $\boldsymbol{K}_{x}$ and $\boldsymbol{K}_{y}$
associated with the largest eigenvalues.
\item Compute the SVD decomposition: $
\boldsymbol{W}_{x}^{T}\boldsymbol{W}_{y}=\boldsymbol{Q}\boldsymbol{\Gamma}\boldsymbol{R}^{T}\in\mathbb{R}^{d\times d}$
\item Compute $
\boldsymbol{U}=\frac{1}{\sqrt{2}}\left[\begin{matrix}\boldsymbol{W}_{x} & \boldsymbol{W}_{y}\end{matrix}\right]\left[\begin{matrix}\boldsymbol{Q} & \boldsymbol{Q}\\
\boldsymbol{R} & -\boldsymbol{R}
\end{matrix}\right]\left[\begin{matrix}\boldsymbol{I}+\boldsymbol{\Gamma} & \boldsymbol{0}\\
\boldsymbol{0} & \boldsymbol{I}-\boldsymbol{\Gamma}
\end{matrix}\right]^{-\frac{1}{2}}\in\mathbb{R}^{N\times2d}
$
\item Set $\boldsymbol{f}_{m}$ to the $m$th column of $\boldsymbol{U}$.
\end{enumerate}
\caption{Jointly smooth functions from 2 observations}
    \label{alg:2views}
\end{algorithm}

\begin{algorithm}[t]
\textbf{\uline{Input}}\textbf{:} $K$ observations $\big\{ \boldsymbol{x}_{i}^{(1)},\boldsymbol{x}_{i}^{(2)},\dots\boldsymbol{x}_{i}^{(K)}\big\} _{i=1}^{N}$
where $\boldsymbol{x}_{i}^{(k)}\in\mathbb{R}^{d_{k}}$.

\textbf{\uline{Output}}\textbf{:} $M$ jointly smooth functions $\{ \boldsymbol{f}_{m}\in\mathbb{R}^{N}\} _{m=1}^{M}$.
\begin{enumerate}
\item For each observation set $\big\{ \boldsymbol{x}_{i}^{(k)}\big\} _{i=1}^{N}$
compute the kernel:
\[\boldsymbol{K}_{k}\left[i,j\right]=\exp\left(-\frac{\big\Vert \boldsymbol{x}^{(k)}_{i}-\boldsymbol{x}^{(k)}_{j}\big\Vert _{2}^{2} }{2\sigma_{k}^{2}}\right)
\]
\item Compute $\boldsymbol{W}_{k}\in\mathbb{R}^{N\times d}$, the first
$d$ eigenvectors of $\boldsymbol{K}_{k}$.
\item Set $\boldsymbol{W}\coloneqq\left[\boldsymbol{W}_{1},\boldsymbol{W}_{2},\dots,\boldsymbol{W}_{K}\right]\in\mathbb{R}^{N\times Kd}$
\item Compute the SVD decomposition: $\boldsymbol{W}=\boldsymbol{U}\boldsymbol{\Sigma}\boldsymbol{V}^{T}$
\item Set $\boldsymbol{f}_{m}$ to be the $m$th column of $\boldsymbol{U}$.
\end{enumerate}
\caption{Jointly smooth functions from $K$ observations}
    \label{alg:multi}
\end{algorithm}

\subsection{Learning the common manifold}
\sloppy Similar to the idea of Alternating Diffusion Maps~\cite{lederman2018learning}, we can use jointly smooth functions as embedding coordinates for the common manifold between the sensor data.
The minimal number of embedding coordinates depends on the intrinsic dimension (bounds given by the theorem of Whitney~\cite{whitney-1936}) and the topology of the common manifold.
For example, a circle with intrinsic dimension of one requires at least two coordinates to embed it in a Euclidean space.
In our case, these coordinates would correspond to two jointly smooth functions.

In practice, we first compute all jointly smooth functions that have a score larger than $E_0$ (see Section~\ref{sub:ChooseM}). Then, we use nonlinear manifold learning algorithms such as Diffusion Maps~\cite{coifman2006diffusion} in combination with eigenvector selection algorithms such as Local Linear Regression~\cite{Dsilva-2018,chen-2019} to select good embedding coordinates.
We demonstrate this approach on an example in Section~\ref{sec:highdim}.

\subsection{Analysis of computational complexity and efficient implementation}
    \label{sec:NA}

Algorithms~\ref{alg:2views} and \ref{alg:multi} are kernel methods that employ as many eigendecompositions as there are data sets, and subsequently involve one singular value decomposition (SVD) to extract jointly smooth features. For data sets with a  large number $N$ of data points, numerical methods implementing the algorithms have to be chosen carefully. Otherwise, a brute-force approach  scales with a computational complexity of $O\left(N^3\right)$ for the eigendecompositions and SVD computations, and a memory requirement of $O\left(N^2\right)$ due to the storage of the full kernel matrices for each data set.
In general, for dense matrices and a complete set of eigenvectors, the best achievable computational complexity is $O\left(N^{\omega}\right),$ $\omega\approx 2.376$~\cite{Coppersmith-1990}. However, we can exploit more structure and sparse matrices to bring the computational and memory cost down to an almost linear scaling. 
\sloppy With efficient distance computations, sparse matrix storage, and numerical eigensolvers that can handle such data structures, we can reduce the computational efficiency. Let $d$ be the number of eigenvectors we compute per kernel and let $M$ be the number of jointly smooth functions, then the computational time scales with $O\left(N\log(N)\right)$ and the memory requirement can be reduced to $O\left((2d+M) N\right)$ for the final result, with $O\left(dN+d^2\right)$ during runtime of the Iteratively Restarted Arnoldi algorithm~\cite{Sorensen-1997}.
The kernel matrices can be kept sparse by using a cut-off for the exponentially decaying Gaussian kernels, or by employing (continuous) nearest-neighbor kernels~\cite{berry-2019}.
The factor of $O\left(N\log(N)\right)$ occurs in the computation of the ($k$-nearest neighbor) kernel matrix, not in the eigenvector computations. The latter scale with $O(N)$, as long as the matrix-vector products can be computed in linear time (which is possible because of sparsity) and the number of Lanczos vectors is kept fixed, see~\cite[p.135]{Sorensen-1997}. 
Note that there are many other approaches to good computational scaling of kernel methods~\cite{mcQueen-2016,Shen-2020}.

 The following numerical methods were applied:
 \begin{enumerate}
 \item We compute the continuous nearest-neighbor kernel~\cite{berry-2019} (with $k=25$ and $\delta=1.0$) instead of the squared-exponential kernel. The limiting operator is the same (i.e. in the large data limit, the function spaces are the same), but the kernel can be computed much faster, using SciPy's k-d-tree implementation to obtain the nearest neighbors of each data point.
 \item We use the efficient, sparse eigensolver for Hermitian matrices implemented in SciPy 
  to obtain $d=100$ eigenvectors of each kernel.
 \item We then extract $M=10$ jointly smooth functions $\boldsymbol{f}_m$ using 
 the sparse implementation of the singular value decomposition, on the merged eigenvectors.
 \end{enumerate}
The time and memory efficiency of these algorithms is demonstrated for the toy problem in Section~\ref{sub:toy}.

\section{Extending jointly smooth functions to new data points}
    \label{sec:OOSE}
Consider a new (unseen) set of pairs of observations $\left\{ \left(\boldsymbol{x}_{i}^{\star},\boldsymbol{y}_{i}^{\star}\right)\right\} _{i=1}^{N^{\star}}$. We wish to estimate the evaluation of the computed jointly smooth functions in Algorithm \ref{alg:2views} (or in Algorithm \ref{alg:multi}) on this new set, namely, to estimate new vector coordinates $\boldsymbol{f}^{\star}_{m}\left[i\right]$ or equivalently the values $f_{m}^{x}\left(\boldsymbol{x}^{\star}_{i}\right)$ and $f_{m}^{y}\left(\boldsymbol{y}^{\star}_{i}\right)$.
Clearly, we can append the new points to the existing set of points and reapply Algorithm \ref{alg:2views} to the extended set.
In this section, we propose instead a more efficient estimate of $\boldsymbol{f}_{m}^{\star}\left[i\right]$ that can be implemented in an online manner, supporting incoming streaming data.
The proposed extension is based on the Nystr\"{o}m method, which has been extensively used for out-of-sample extension in the context of kernel methods \cite{coifman2006geometric,bengio2004out,fowlkes2001efficient,williams2001using}.
Consider a new pair of data $\left(\boldsymbol{x}_{i}^{\star},\boldsymbol{y}_{i}^{\star}\right)$. Let 
\[
\boldsymbol{\alpha}_{x}^m=\boldsymbol{W}_{x}^{T}\boldsymbol{f}_{m}\in\mathbb{R}^{d\times1}
\]
be the expansion coefficients of $\boldsymbol{f}_{m}$ (obtained by Algorithm \ref{alg:2views}) in the basis $\boldsymbol{W}_{x}$.
Recall that $\boldsymbol{W}_{x}$ consists of eigenvectors of $\boldsymbol{K}_{x}$, as defined in Section \ref{sec:pre}, that is, $\boldsymbol{K}_{x}\boldsymbol{W}_{x}=\boldsymbol{W}_{x}\boldsymbol{\Lambda}_{x}$. Using the Nystr\"{o}m extension, we extend the eigenvectors by: 
\[
\left[\begin{matrix}\ulcorner &  & \urcorner\\
\,\,\,\,\,\,\,\,\,\,\,\, & \boldsymbol{K}_{x} & \,\,\,\,\,\,\,\,\,\,\,\,\\
\llcorner &  & \lrcorner\\
\text{---}\hspace{-0.2cm} & \boldsymbol{k}_x^{\star} & \hspace{-0.2cm}\text{---}
\end{matrix}\right]\left[\begin{matrix}\\
\,\,\,\,\,\, & \boldsymbol{W}_{x} & \,\,\,\,\,\,\\
\\
\end{matrix}\right]=\left[\begin{matrix}\ulcorner &  & \urcorner\\
\,\,\,\,\,\,\,\,\,\,\,\, & \boldsymbol{W}_{x} & \,\,\,\,\,\,\,\,\,\,\,\,\\
\llcorner &  & \lrcorner\\
\text{---}\hspace{-0.2cm} & \boldsymbol{w}_x^{\star} & \hspace{-0.2cm}\text{---}
\end{matrix}\right]\left[\begin{matrix}\\
\,\,\,\,\,\, & \boldsymbol{\Lambda}_{x} & \,\,\,\,\,\,\\
\\
\end{matrix}\right],
\]
where $\left[\begin{matrix}\text{---}\hspace{-0.2cm} & \boldsymbol{k}_x^{\star} & \hspace{-0.2cm}\text{---}\end{matrix}\right]\in\mathbb{R}^{1\times d}$ and $\boldsymbol{k}_x^{\star}\left[j\right]=\exp\left(-\frac{\left\Vert \boldsymbol{x}^{\star}_{i}-\boldsymbol{x}_{j}\right\Vert _{2}^{2}}{2\sigma_{x}^{2}}\right)$ is computed from the existing and new data.
In other words, the new coordinates of the basis vectors (eigenvectors) are given by
\[
\boldsymbol{w}_x^{\star}={\boldsymbol{k}}_{x}^{\star}\boldsymbol{W}_{x}\boldsymbol{\Lambda}_{x}^{-1}\in\mathbb{R}^{1\times d},
\]
and the induced out-of-sample extension is given by
\begin{equation*}
\boldsymbol{f}^{\star}_m\left[i\right]=\boldsymbol{w}_x^{\star}\boldsymbol{\alpha}_{x}^m.
\end{equation*}
Note that the procedure above is based on the computation of $\boldsymbol{k}_x^*$ using $\boldsymbol{x}_i^*$ and $\{\boldsymbol{x}_i\}_{i=1}^N$. 
Similarly, this procedure can be implemented based on $\boldsymbol{y}^{\star}_{i}$ and $\{\boldsymbol{y}_i\}_{i=1}^N$.
Combining these two alternatives based on the constraint that $\boldsymbol{f}^{\star}_{m}\left[i\right]=f_{m}^{x}\left(\boldsymbol{x}^{\star}_{i}\right)=f_{m}^{y}\left(\boldsymbol{y}^{\star}_{i}\right)$ gives rise to the following extension
$$\boldsymbol{f}^{\star}_{m}\left[i\right]=\frac{1}{2}\left(\boldsymbol{w}_{x}^{\star}\boldsymbol{\alpha}_{x}^{m}+\boldsymbol{w}_{y}^{\star}\boldsymbol{\alpha}_{y}^{m}\right).$$
Considering other combinations, for instance, taking a weighted mean value, will be subject of future work.
The entire out-of-sample extension algorithm is given in Algorithm \ref{alg:OOSE}. 

\begin{algorithm}[t]
\textbf{\uline{Input}}\textbf{:} A set of new observations $\left\{ \left({\boldsymbol{x}}^{\star}_{i},{\boldsymbol{y}}^{\star}_{i}\right)\right\}_{i=1}^{N^{\star}}$.

\textbf{\uline{Output}}\textbf{:} The out-of-sample extension $f_{m}^{x}({\boldsymbol{x}}^{\star}_{i})$
and $f_{m}^{y}({\boldsymbol{y}}^{\star}_{i})$ for all $i$.
\begin{enumerate}
\item Compute the coefficients: $\boldsymbol{\alpha}_{x}^{m}=\boldsymbol{W}_{x}^{T}\boldsymbol{f}_{m}\in\mathbb{R}^{d\times 1}$

\item Extend the eigenvectors: ${\boldsymbol{W}}^{\star}_{x}={\boldsymbol{K}}^{\star}_{x}\boldsymbol{W}_{x}\boldsymbol{\Lambda}_{x}^{-1}\in\mathbb{R}^{N^{\star}\times d}$

where ${\boldsymbol{K}}^{\star}_{x}\in\mathbb{R}^{N^{\star}\times N}$ such
that ${\boldsymbol{K}}^{\star}_{x}\left[i,j\right]=\exp\left(-\frac{\Vert {\boldsymbol{x}}^{\star}_{i}-\boldsymbol{x}_{j}\Vert _{2}^{2}}{2\sigma_{x}^{2}}\right)$
\item Repeat Steps 1 and 2 for the $\left\{ \boldsymbol{y}^{\star}_{i}\right\} $ observations.
\item Set ${\boldsymbol{f}}^{\star}_{m}=\frac{1}{2}({\boldsymbol{W}}^{\star}_{x}\boldsymbol{\alpha}_{x}^{m}+{\boldsymbol{W}}^{\star}_{y}\boldsymbol{\alpha}_{y}^{m})\in\mathbb{R}^{N^{\star}}$, such that $f_{m}^{x}({\boldsymbol{x}}^{\star}_{i})=f_{m}^{y}({\boldsymbol{y}}^{\star}_{i})={\boldsymbol{f}}^{\star}_{m}[i]$.
\end{enumerate}
\caption{Out-of-sample extension}
    \label{alg:OOSE}
\end{algorithm}

\section{Experimental Results}
    \label{sec:Exp}
Our source code for all experiments will be publicly available upon acceptance.
\subsection{Toy Problem}
    \label{sub:toy}
We revisit the toy problem presented in Section \ref{sec:ProblemFormulation} and start by giving more details. 
Consider triplets $\left(z_{i},\epsilon_{i},\eta_{i}\right)\sim U\left[0,1\right]^{3}$ that are i.i.d. and uniformly distributed in the unit cube.
The $2$D spiral in $\mathbb{R}^2$ is given by  
\begin{equation}
    \label{eq:spiral} \renewcommand*{\arraystretch}{1.1}
\boldsymbol{x}_{i}=g(z_{i},\epsilon_{i})=\left[\begin{matrix}\left(\frac{3}{2}\epsilon_{i}+\frac{z_{i}}{3}+\frac{2}{3}\right)\cos(4\pi\epsilon_{i})\\
\left(\frac{3}{2}\epsilon_{i}+\frac{z_{i}}{3}+\frac{2}{3}\right)\sin(4\pi\epsilon_{i})
\end{matrix}\right],
\end{equation}
such that $z_i$ controls the width of the ribbon. The $2$-torus in $\mathbb{R}^3$ is given by 
\begin{equation}
    \label{eq:torus} \renewcommand*{\arraystretch}{1.1}
\boldsymbol{y}_{i}=h(z_{i},\eta_{i})=\left[\begin{matrix}\left(1+\frac{1}{3}\cos(2\pi z_{i})\right)\cos(2\pi\eta_{i})\\
\left(1+\frac{1}{3}\cos(2\pi z_{i})\right)\sin(2\pi\eta_{i})\\
\frac{1}{3}\sin(2\pi z_{i})
\end{matrix}\right],
\end{equation}
such that $z_i$ controls the smaller angle.
Note that $z_i$ is a common variable observed by both observation functions $g$ and $h$, whereas $\epsilon_i$ and $\eta_i$ are variables captured by only a single observation.

\sloppy We generate the set $\{ (z_{i},\epsilon_{i},\eta_{i} ) \} _{i=1}^{N}$ consisting of $4,100$ realizations of the triplets $(z_{i},\epsilon_{i},\eta_{i}) \sim U[0,1]^{3}$.  We keep $N^{\star}=100$ points aside for the out-of-sample extension validation and construct the two views $\boldsymbol{x}_{i}=f(z_{i},\epsilon_{i} )$ and $\boldsymbol{y}_{i}=g (z_{i},\eta_{i})$ as in \eqref{eq:spiral} and \eqref{eq:torus} based on the remaining $N=4,000$ points.
Let $\boldsymbol{K}_{x}$ and $\boldsymbol{K}_{y}$ be the kernels associated with the observations $\{ \boldsymbol{x}_{i}\in\mathcal{M}_{x} \} _{i}$ and $\{ \boldsymbol{y}_{i}\in\mathcal{M}_{y} \} _{i}$, respectively.
We apply Algorithm~\ref{alg:2views}, where $\sigma_x$ and $\sigma_y$ are set to be $30\%$ of the median of the pairwise Euclidean distances, that is $\sigma_{x}=0.3\cdot\text{median}(\{ \Vert \boldsymbol{x}_{i}-\boldsymbol{x}_{j}\Vert _{2}\} _{i,j})$, and similarly for $\sigma_y$.
We also set $d=\frac{N}{4}=1,000$ and we report that any value between $500 - 2,000$ provides similar empirical results.
Figure \ref{fig:ToyNCCA}(a) presents the two manifolds.
At the top row of Figure \ref{fig:ToyNCCA}(b), we depict the top three (non-trivial) jointly smooth functions obtained by Algorithm \ref{alg:2views}, and the out-of-sample extension obtained by Algorithm \ref{alg:OOSE}. The bottom row of Figure \ref{fig:ToyNCCA}(b) displays the top three (non-trivial) left ($\boldsymbol{\phi}_i$) and right ($\boldsymbol{\psi}_i$) singular vectors obtained by NCCA \cite{michaeli2016nonparametric}.
We observe that Algorithm \ref{alg:2views} provides functions, which represent the common variable $z$ more accurately in comparison to NCCA.

\begin{figure}[h]
    \centering
    \includegraphics[width=1\columnwidth]{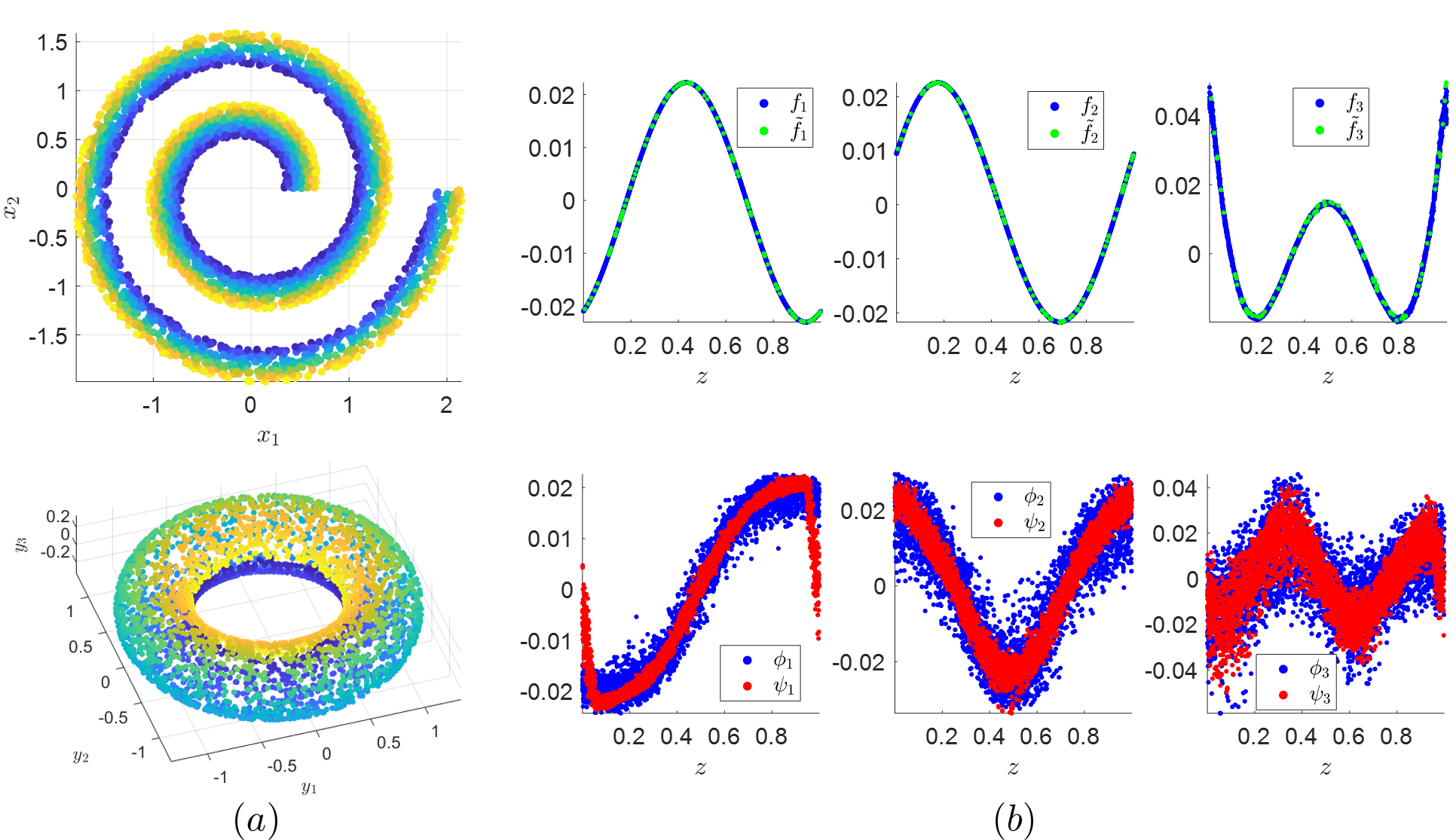}
    \caption{(a) The two manifolds colored with respect to the common variable $z$. The blue points at the top row of (b) are scatter of the top three (non-trivial) functions obtained by Algorithm \ref{alg:2views} plotted against the common variable $z$. The green points are obtained by the out-of-sample extension  Algorithm \ref{alg:OOSE}.
    The bottom row of (b) consists of the top three (non-trivial) left ($\boldsymbol{\phi}_i)$) and right ($\boldsymbol{\psi}_i)$) functions (singular vectors) obtained by NCCA plotted against the common variable $z$. Recall that in NCCA the left singular vectors correspond to the spiral $\mathcal{M}_x$ and the right singular vectors correspond to the torus $\mathcal{M}_y$.}
    \label{fig:ToyNCCA}
\end{figure}

Figure \ref{fig:ToyM} depicts the top $8$ jointly smooth functions obtained by Algorithm~\ref{alg:2views}. Above each plot we present the smoothness value $\Vert \boldsymbol{W}_{x}^{T}\boldsymbol{f}_{i}\Vert _{2}^{2}$ and compare it to the threshold $E_0 =0.9558$ obtained using the jackstraw method as discussed in Section \ref{sub:ChooseM}. We observe that indeed the threshold allows us to accurately detect the number of components $M$ correlated with the common variable $z$. 

\begin{figure}[h]
    \centering
    \includegraphics[width=1\columnwidth]{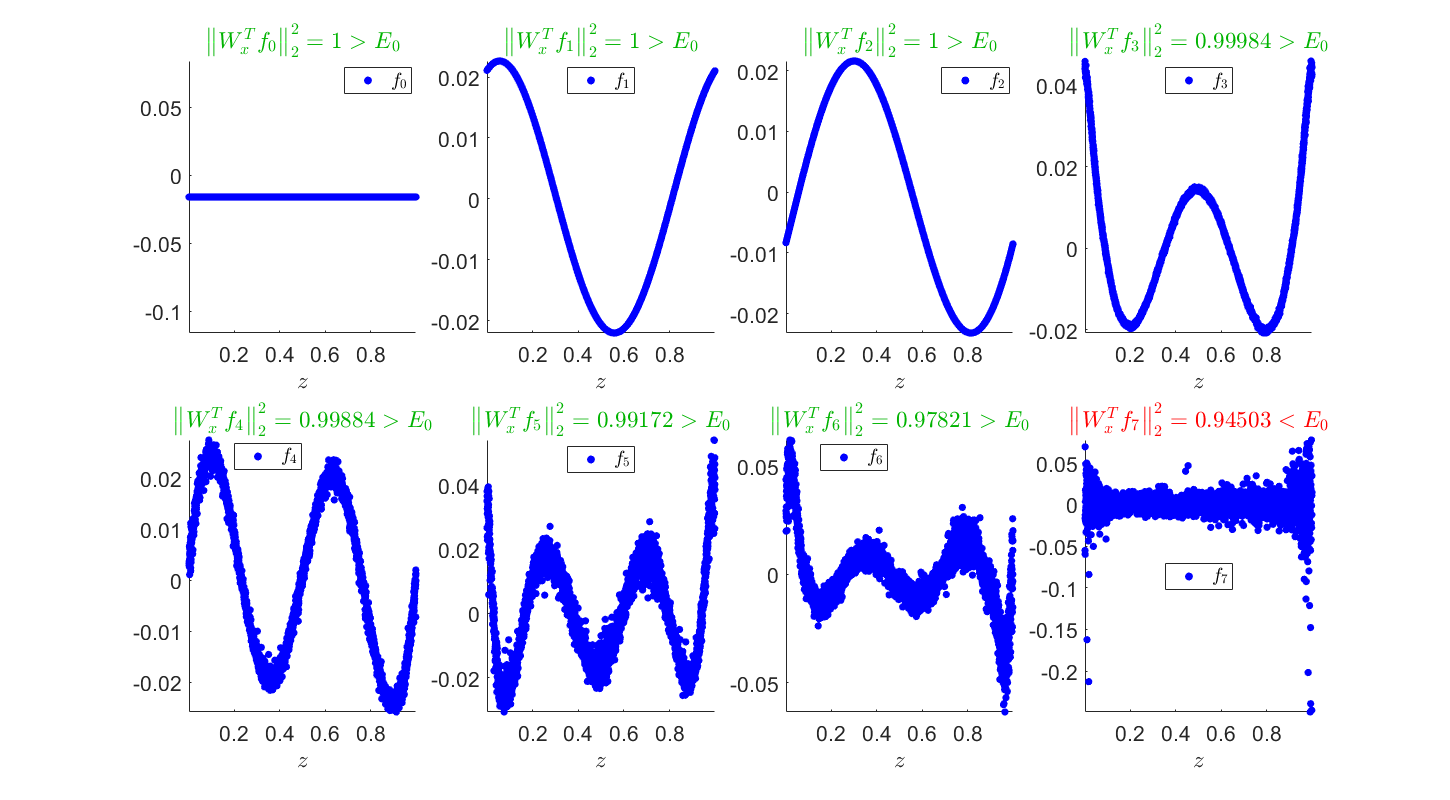}
    \caption{The top $8$ jointly smooth functions obtained by Algorithm \ref{alg:2views} plotted against the common variable $z$. Above each plot we depict the smoothness value $\left\Vert \boldsymbol{W}_{x}^{T}\boldsymbol{f}_{i}\right\Vert _{2}^{2}$ and compare it to the threshold $E_0 =0.9558$ obtained by the jackstraw method described in Section \ref{sub:ChooseM}.}
    \label{fig:ToyM}
\end{figure}

In Figure~\ref{fig:efficiency_total_seconds}, we demonstrate that the time and memory complexity of the algorithm actually follows the predicted curves. The data was created for the given toy problem, with varying number of points.
 In Figure~\ref{fig:efficiency_eigenvectors_seconds}, the number of points in the data set was fixed to $N=50000$, and we demonstrate how the algorithms scale when the number of eigenvectors $d$ per kernel is changed. As expected, both memory and computational time scale linearly.
 \begin{figure}[h]
 \centering
\begin{tabular}{cc}
 \includegraphics[width=0.45\textwidth]{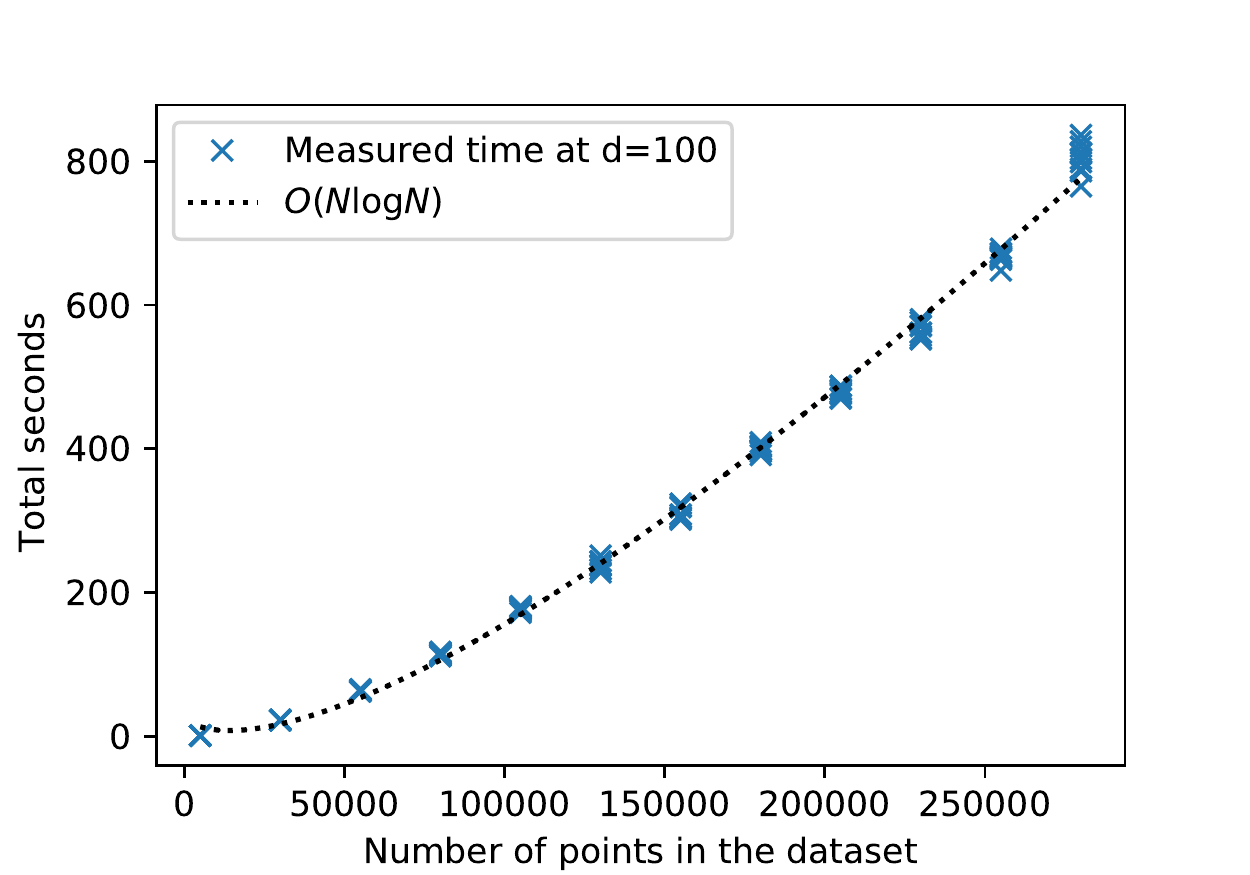}&
 \includegraphics[width=0.45\textwidth]{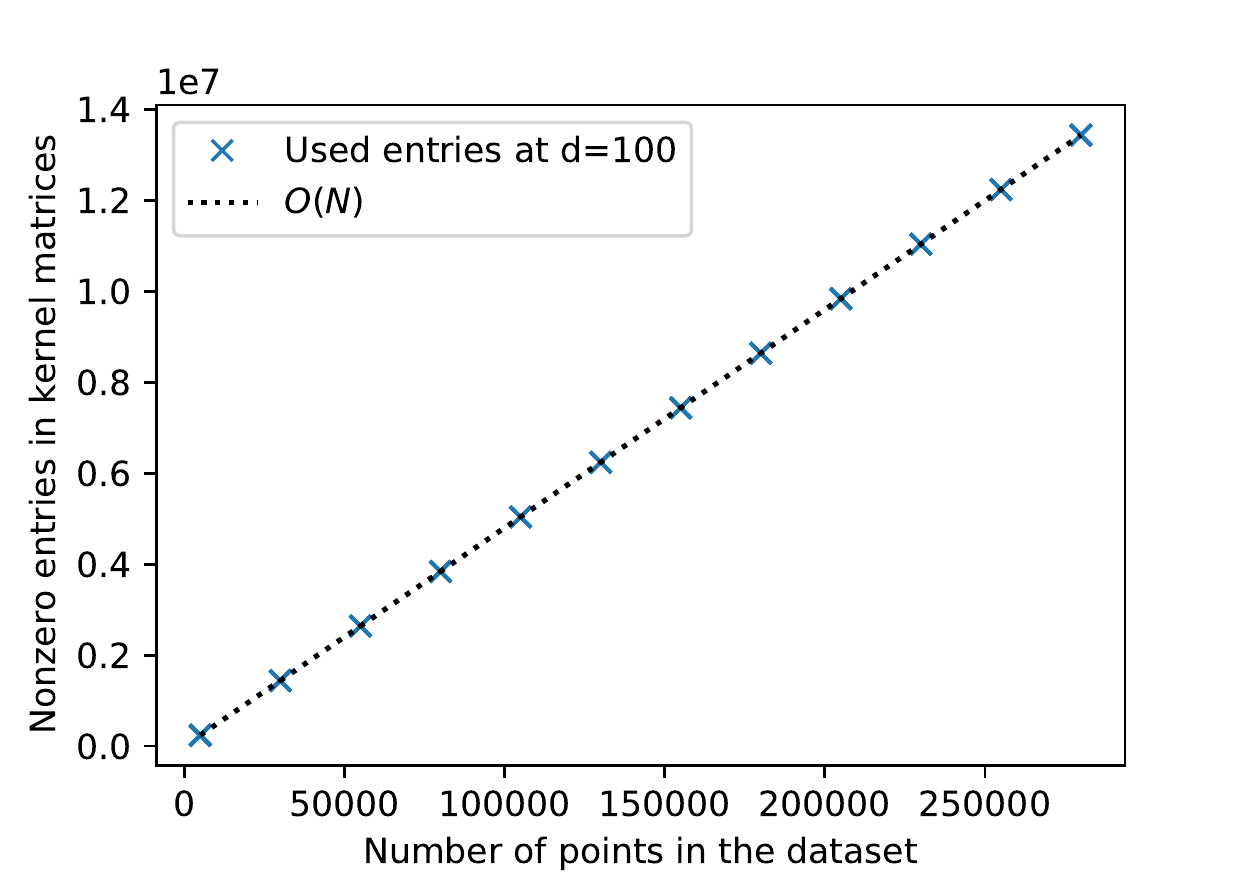}
\\
(a)&(b)
\end{tabular}
 \caption{\label{fig:efficiency_total_seconds}Scaling of (a) computation time and (b) memory, with the number of points $N$ in the dataset changing at a constant number of eigenvectors $d=100$, for algorithm~\ref{alg:2views} with efficient kernel and SVD algorithms. The actual curve fitted to the data in panel (a) is $0.0013455 N log(N) - 0.0141995 N + 26.9836555$. Every experiment was repeated 10 times.}
 \end{figure}
 
 \begin{figure}[h]
 \centering
\begin{tabular}{cc}
 \includegraphics[width=0.45\textwidth]{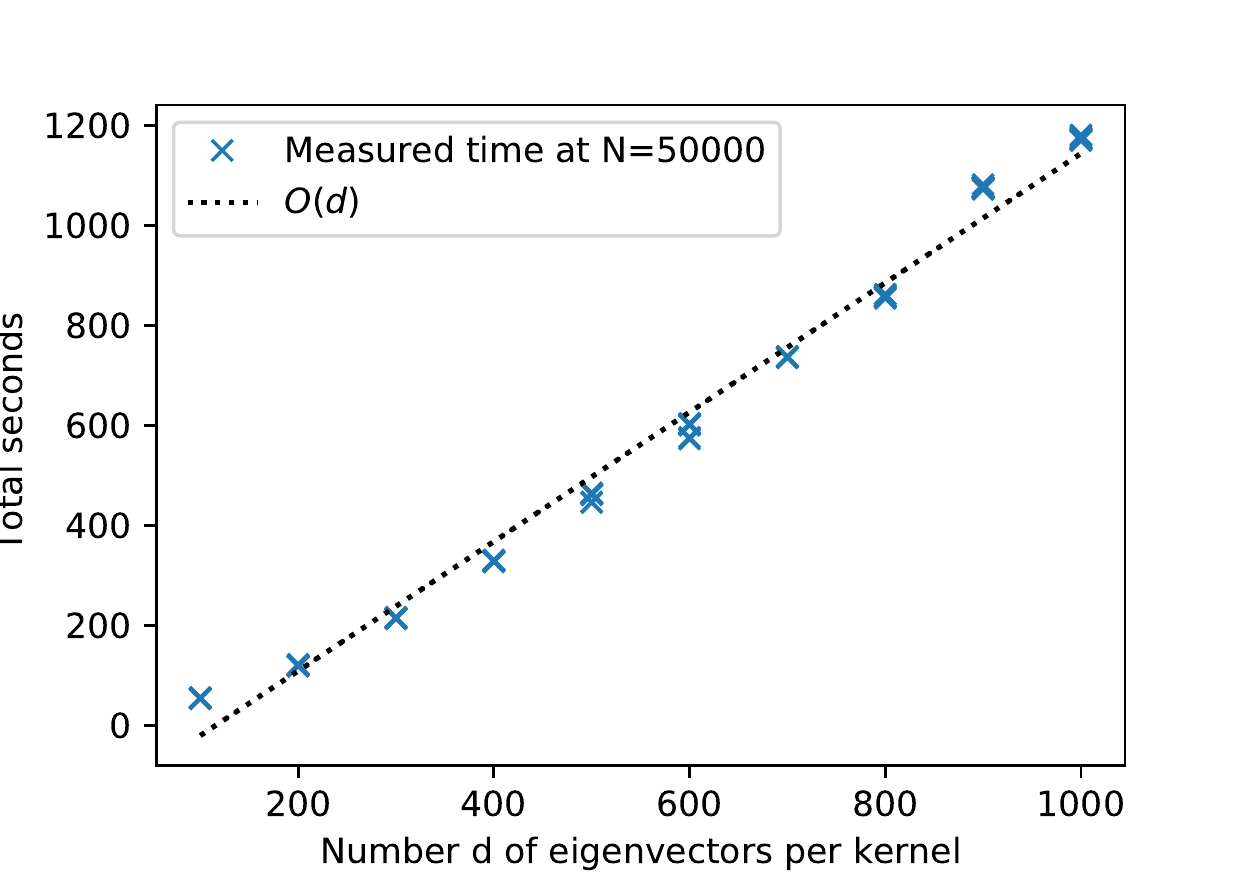}&
 \includegraphics[width=0.45\textwidth]{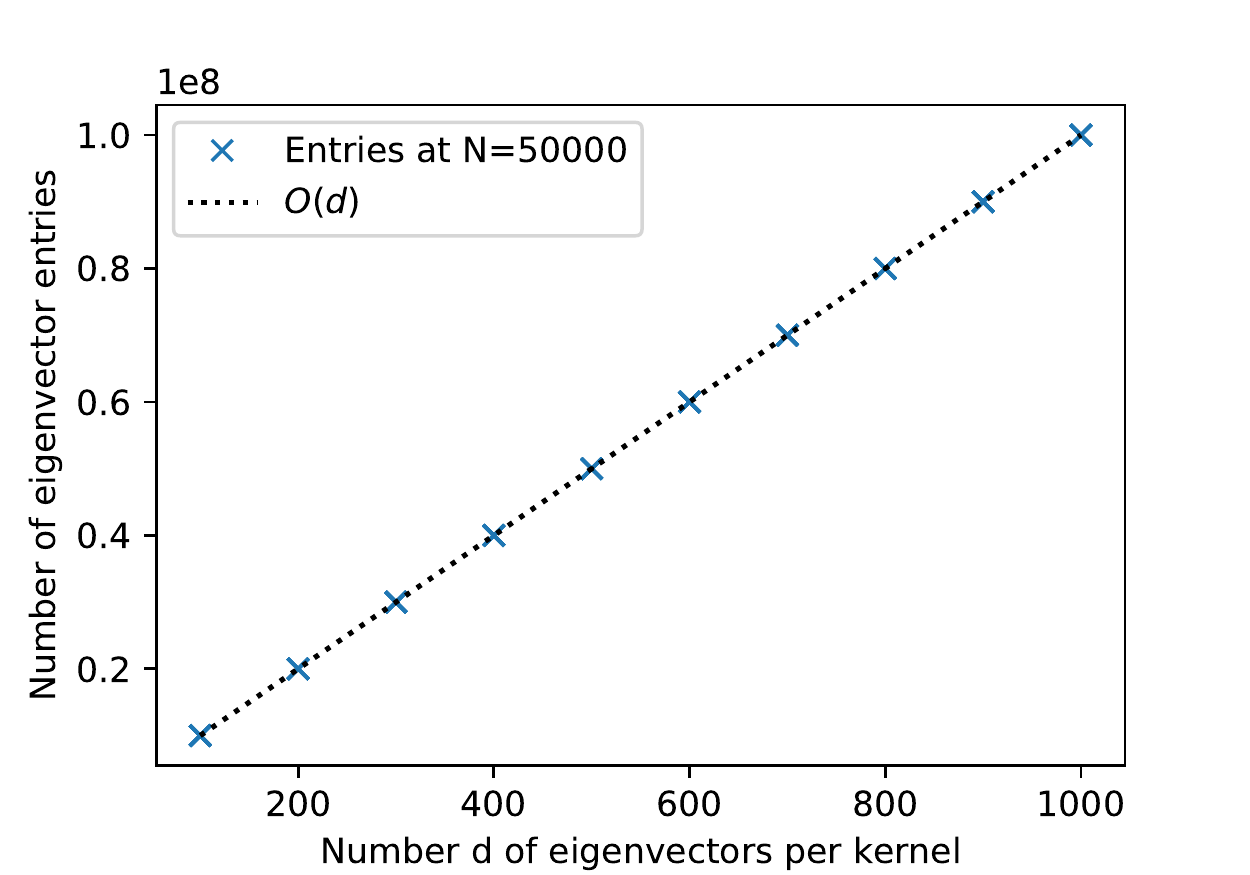}
\\
(a)&(b)
\end{tabular}
 \caption{\label{fig:efficiency_eigenvectors_seconds}Scaling of (a) computation time and (b) memory, with the number of eigenvectors $d$ per kernel at $N=50000$, for algorithm~\ref{alg:2views} with efficient kernel and SVD algorithms. The actual curve fitted to the data in panel (a) is $1.2945669 d -150.0280786$. Every experiment was repeated 10 times.}
 \end{figure}

\subsection{Sleep Stage Identification}
    \label{sub:sleep}
We apply Algorithm \ref{alg:multi} (multi-view) to real physiological signals, addressing the problem of sleep stage identification.
The data are available online~\cite{goldberger2000physiobank} and described in detail in \cite{terzano2002atlas}. 
The data contain multimodal recordings from several subjects, where each subject is recorded for about $10$ hours during a single night sleep.
The data were analyzed by a human expert and divided into six sleep stages: awake, Rapid Eye Movement (REM), stage 1 (shallow sleep), stage 2, stage 3, and stage 4 (deep sleep).
The types of sensors used for recording vary between the different subjects. Here, we use a subset of six sensors common to all reported subjects: four electroencephalogram (EEG) channels, one Electromyography (EMG) channel, and one Electrocardiography (ECG) channel, sampled at $512$~Hz.

Let $\boldsymbol{K}_{i}^{\left(k\right)}$ be the kernel constructed from the $i$th channel of the $k$th subject. For details on the construction of the kernels from the measured data, see the SM.
%
%
We apply Algorithm \ref{alg:multi} to the six kernels with $d = 1,000$ and obtain the top $M = 20$ common functions $\boldsymbol{f}_m$. We note that about $50$ functions could be used but we take only the top $20$ to make a fair comparison with KCCA, which attains only $20$ relevant canonical directions.   

Figure \ref{fig:Sleep}(a) presents the scatter of the top two (non-trivial) functions obtained by Algorithm \ref{alg:multi} applied to Subject 2. We observe that solely from the top two jointly smooth functions, we obtain a meaningful representation of the sleep stages. Namely, shallow sleep stages reside on the left side of the scatter plot and deep sleep stages reside to the right. We quantify the (unsupervised) separation obtained by the top $20$ jointly smooth functions using a kernel SVM classifier equipped with a Gaussian kernel. The $10$-fold cross validation confusion matrix is displayed in Figure \ref{fig:Sleep}(b). 
For comparison, we apply a multi-view KCCA \cite{yu2013kernel} to the same six kernels. Figure \ref{fig:Sleep}(c) presents the scatter of the top two (non-trivial) eigenvectors obtained by KCCA applied to Subject 2. Figure \ref{fig:Sleep}(d) displays the $10$-fold confusion matrix obtained using a kernel SVM applied to the top $20$ eigenvectors attained by KCCA. As can be seen both qualitatively and quantitatively, the jointly smooth functions obtained using Algorithm \ref{alg:multi} better parametrize the sleep stages. We also compare the classification results to: (i) the representation obtained by each kernel (sensor) separately, and (ii) the representation obtained by concatenating all six feature vectors into one vector. We repeat this experiment for $4$ subjects. Table \ref{tbl:sleep} summarizes the results. In each row, we mark in bold the the highest obtained accuracy, which is consistently achieved by the proposed algorithm.
We note that that we compare to KCCA rather than NCCA because NCCA is restricted to two views.

\begin{figure}[h]
    \centering
    \includegraphics[width=1\columnwidth]{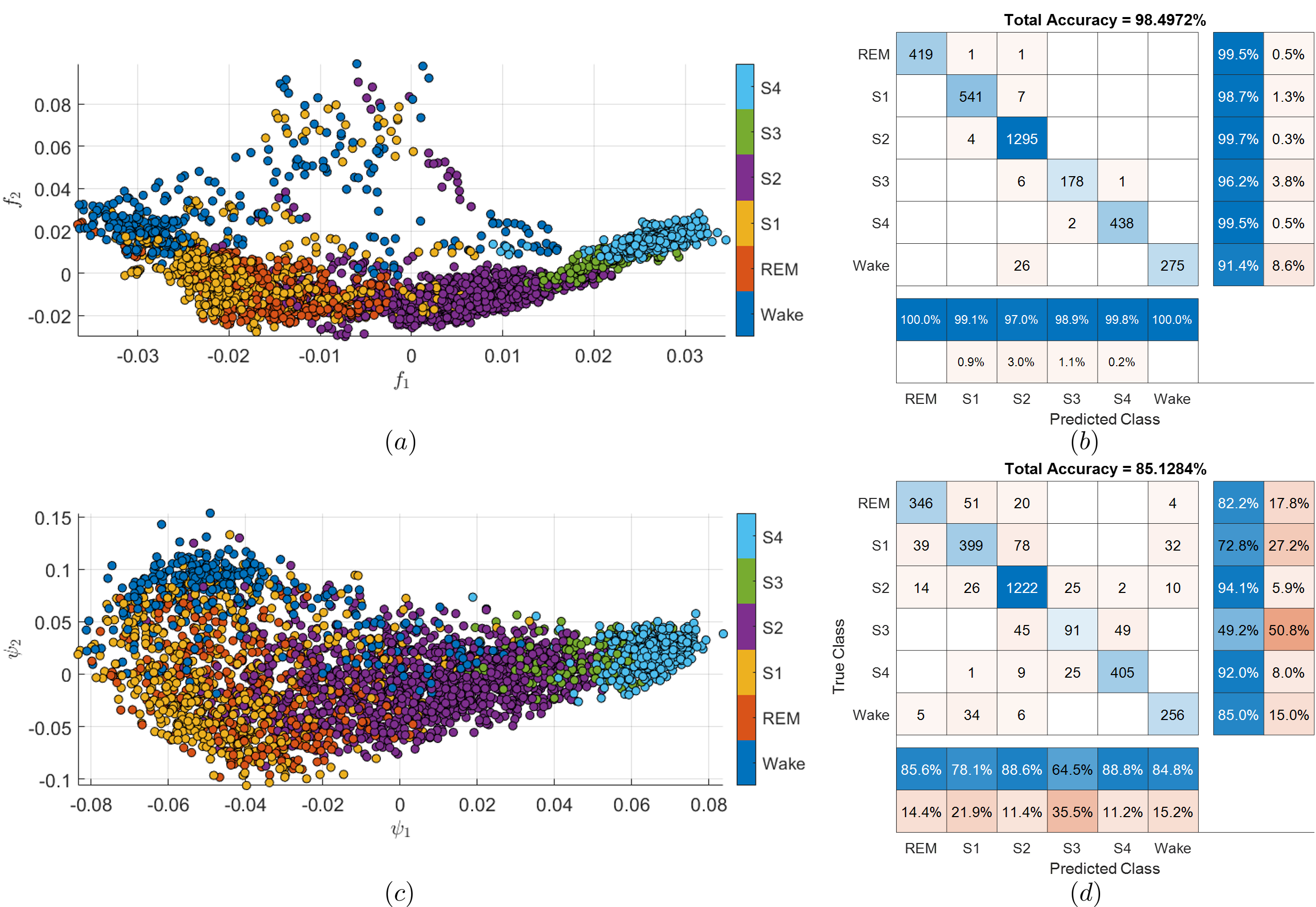}
    \caption
    {
        (a) A scatter plot of the top two jointly smooth functions obtained by Algorithm \ref{alg:multi}. The points are colored according to the different sleep stages as identified by a human expert.
        (b) The confusion matrix obtained by a 10-fold kernel-SVM classifier applied to the top $M=20$ jointly smooth functions obtained by Algorithm \ref{alg:multi}.
        (c) A scatter plot of the top two principal directions obtained by a multi-view KCCA equipped with the same kernels as in (a) and colored according to the sleep stage.
        (d) The confusion matrix obtained by a 10-fold kernel-SVM classifier applied to the top $M=20$ principal directions obtained by multi-view KCCA.
    } 
    \label{fig:Sleep}
\end{figure}

\bgroup
\setlength{\tabcolsep}{3pt}
\renewcommand{\arraystretch}{1.3}
\begin{table}[t]
\caption{Sleep stage identification accuracy {[}\%{]}} 
    \label{tbl:sleep}
\begin{centering}
\begin{tabular}{cccccccccc}
\hline 
\rowcolor{TitleGray}
\textbf{} & \textbf{EEG1} & \textbf{EEG2} & \textbf{EEG3} & \textbf{EEG4} & \textbf{EMG} & \textbf{ECG} & \textbf{All} & \textbf{KCCA} & \textbf{Alg. \ref{alg:multi}}
\tabularnewline \hline 
\textbf{Subject 1} & 78.73 & 78.54 & 76.03 & 74.77 & 69.34 & 81.71 & 96.27 & 88.61 & \textbf{97.52}
\tabularnewline \rowcolor{RowGray}
\textbf{Subject 2} & 80.18 & 78.02 & 78.33 & 76.61 & 79.05 & 79.21 & 93.17 & 85.12 & \textbf{98.49}
\tabularnewline
\textbf{Subject 3} & 78.68 & 77.61 & 79.74 & 81.53 & 79.84 & 76.30 & 89.67 & 81.59 & \textbf{98.18}
\tabularnewline \rowcolor{RowGray}
\textbf{Subject 4} & 75.39 & 78.87 & 76.55 & 77.52 & 76.83 & 70.76 & 94.99 & 86.60 & \textbf{97.68}
\tabularnewline \hline 
\hline \rowcolor{RowGray}
\textbf{Mean} & 78.25 & 78.26 & 77.66 & 77.61 & 76.27 & 77.00 & 93.53 & 85.48 & \textbf{97.97}
\tabularnewline \hline 
\end{tabular}
\par\end{centering}
\end{table}
\egroup

\subsection{Manifold learning for high-dimensional input data}\label{sec:highdim}

We now demonstrate how to obtain jointly smooth functions and a common latent space in an experiment involving high-dimensional observations.
In this experiment, two cars drive clockwise around a single track, both with their individual but constant speed. One set of observations is gathered with a camera located on one of the cars, another data set is gathered with a camera observing the entire track.
Figure~\ref{fig:carrace_setting} shows the setup for this experiment, together with four example snapshots of the two datasets.

 We simultaneously gather 11100 frames with $400\times 300$ gray-scale pixels from both cameras while the two cars drive at constant speed around the track, with 100 frames per loop for the blue car, and 111 frames per loop for the yellow car.
We pre-process the video data with the following pipeline:
\begin{enumerate}
\item Scale down each frame to $15\times 15$ pixels using \texttt{PIL.Image.NEAREST} resampling from python \texttt{pillow}.
\item Project the data onto its first $10$ principal components.
\item Delay embed each data point by concatenating the subsequent $150$ points of the time series to it, forming new data points of $10\times 151$ dimensions.
\item Flatten the data to obtain vectors of dimension $10\times 151=1510$.
\item Again project the data onto its principal components, and use the first five as a new embedding.
\end{enumerate}
We then apply Algorithm~\ref{alg:2views} to the two datasets to obtain ten jointly smooth functions (see Figure~\ref{fig:carrace_positions_evecs}).
Applying Diffusion Maps on these ten jointly smooth features reveals the periodic and intrinsically one-dimensional structure of the common part between both video streams. The structure originates from the blue car looping around the track, which is common to both sets of observations.
The jointly smooth functions (or, alternatively, the Diffusion Maps embedding) can be used to segment the blue car in the overhead camera images, by simply averaging all collected frames that have similar embedding points (see Figure~\ref{fig:carrace_segmentation}).
\begin{figure}[h]
    \centering
    \includegraphics[height=.19\textheight]{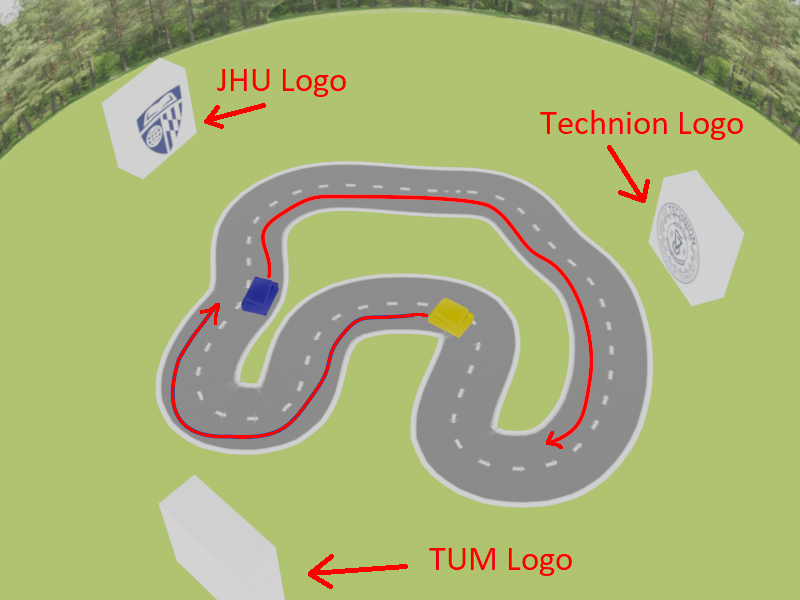}
    \includegraphics[height=.19\textheight]{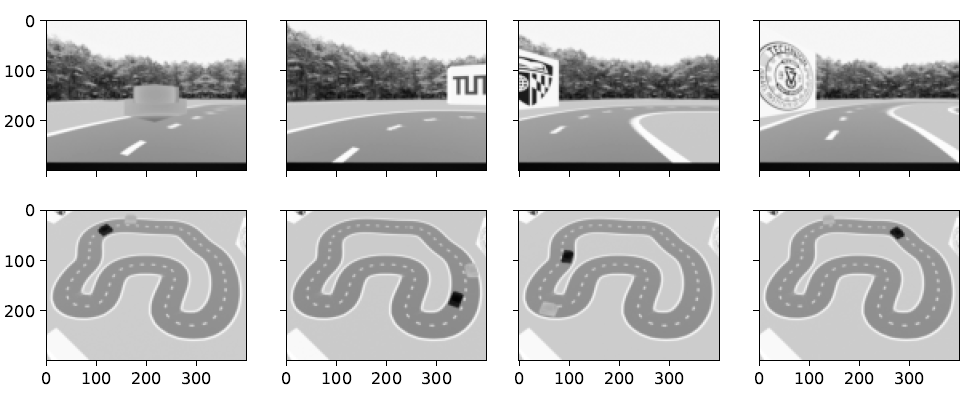}
    \caption{\label{fig:carrace_setting}Left panel: Setting for the car race example. Two cars (blue and yellow) drive clockwise around a single track, both with constant speed (yellow: 100 frames per loop, blue: 111 frames per loop). Right panel: One set of observations is gathered with a camera located on the blue car (grayscale, top row with four example snapshots), another set with a camera observing the entire track (grayscale, bottom row, with corresponding example snapshots).}
\end{figure}
\begin{figure}[h]
    \centering
    \includegraphics[width=1\textwidth]{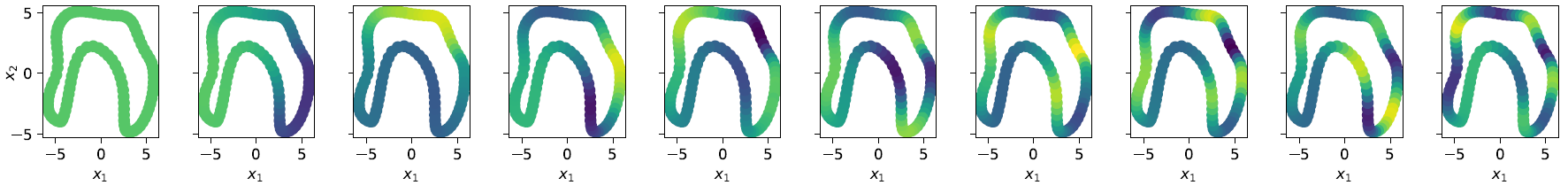}
    \caption{\label{fig:carrace_positions_evecs}Jointly smooth functions plotted over the positions of the blue car. Even though there are several smooth functions available, the underlying manifold is low-dimensional, and manifold learning can reveal this in a subsequent analysis (see Figure~\ref{fig:carrace_evecs_positions}).}
\end{figure}

\begin{figure}[h]
    \centering
    \includegraphics[height=.19\textheight]{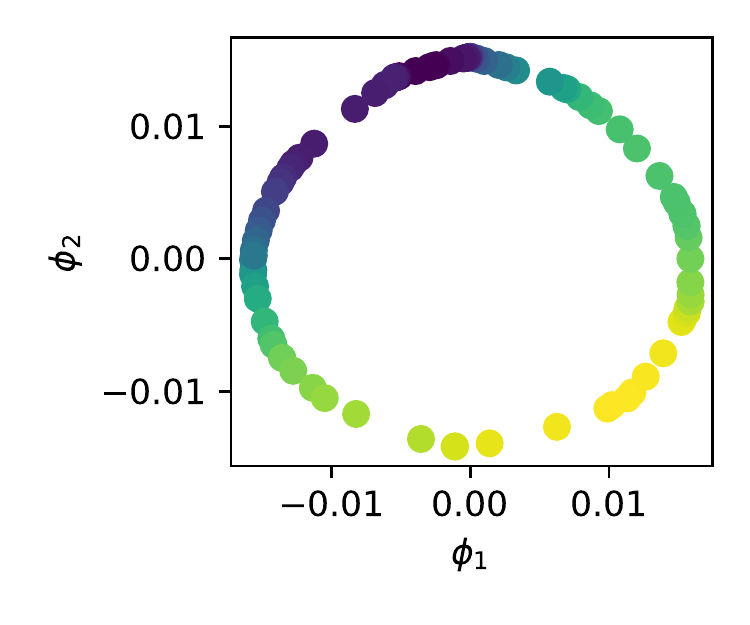}
    \includegraphics[height=.19\textheight]{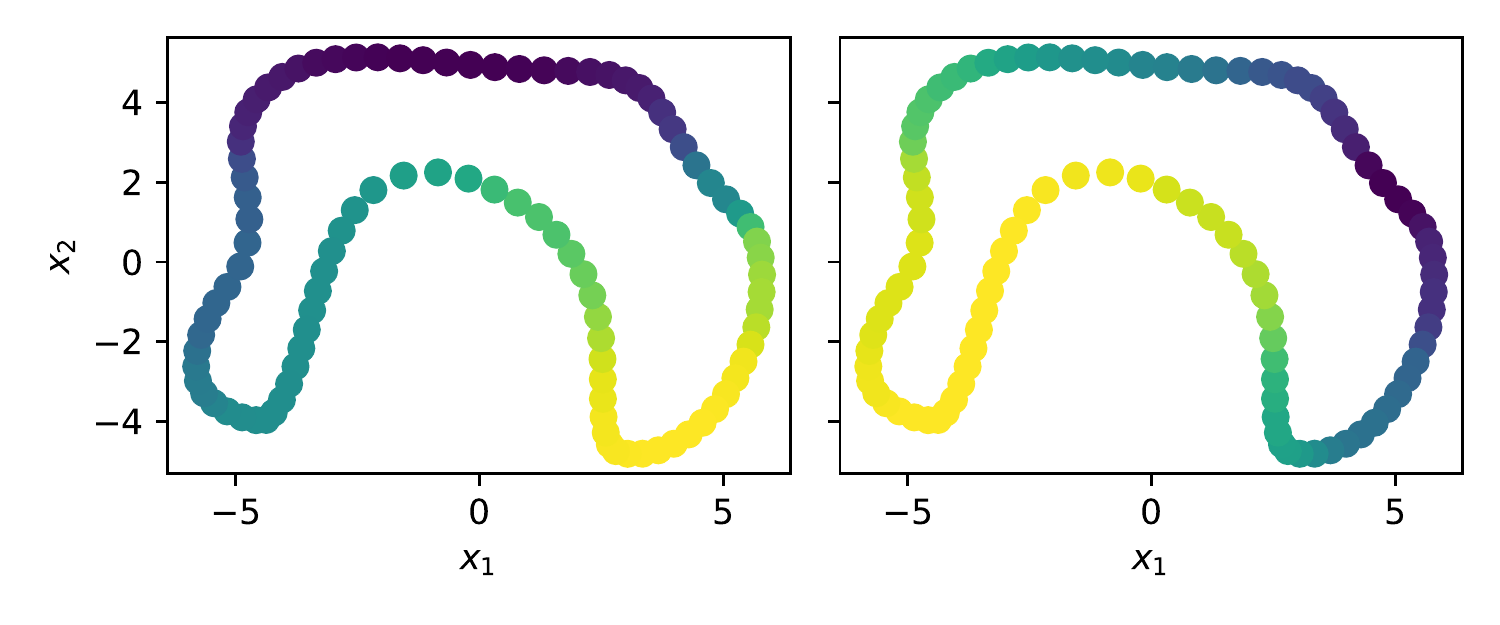}
    \caption{\label{fig:carrace_evecs_positions}Left panel: Diffusion Maps embedding $(\phi_1,\phi_2)$ of the first $10$ jointly smooth functions, colored by the $x_1$-coordinate of the blue car on the track. Right panel: the positions of the blue car (not part of the dataset), colored by the two Diffusion Map coordinates $\phi_1$ and $\phi_2$.}
\end{figure}
\begin{figure}[h]
    \centering
    \includegraphics[height=.19\textheight]{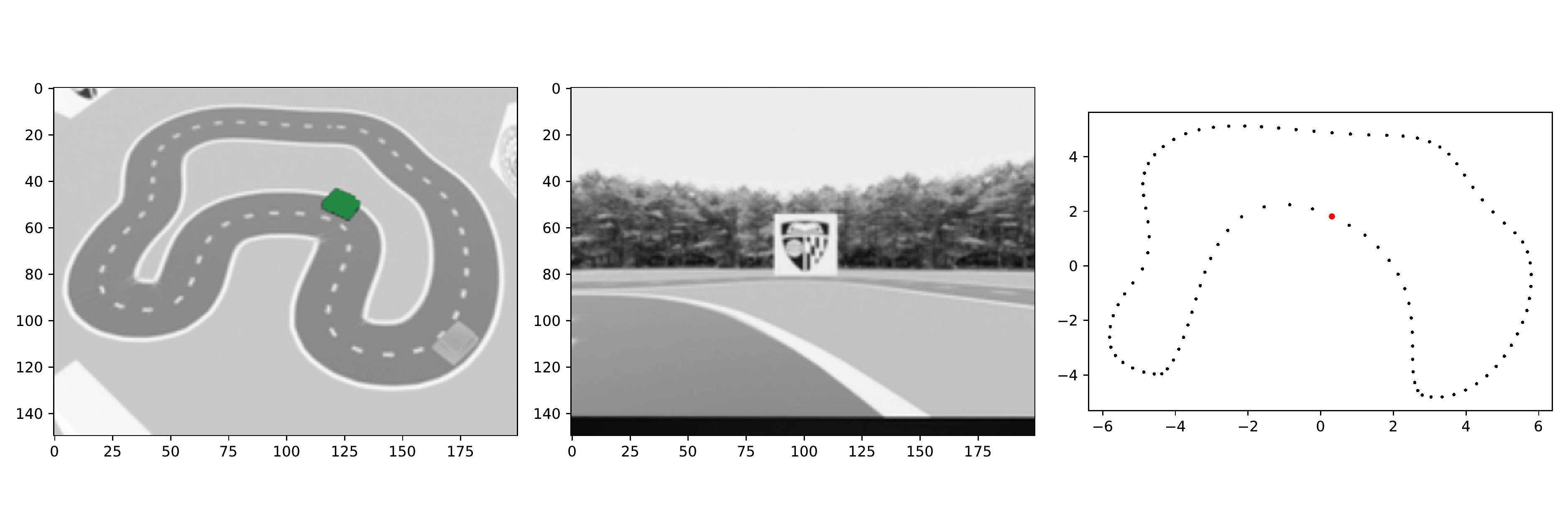}
    \caption{\label{fig:carrace_segmentation}Image segmentation using $k$-nearest neighbor search with jointly smooth functions. The blue car is correctly identified.}
\end{figure}

\subsection{Application to Nonlinear Dynamical Systems Analysis}
    \label{sub:DynamicalSys}

\begin{figure}[h]
\centering
\includegraphics[width=.8\textwidth]{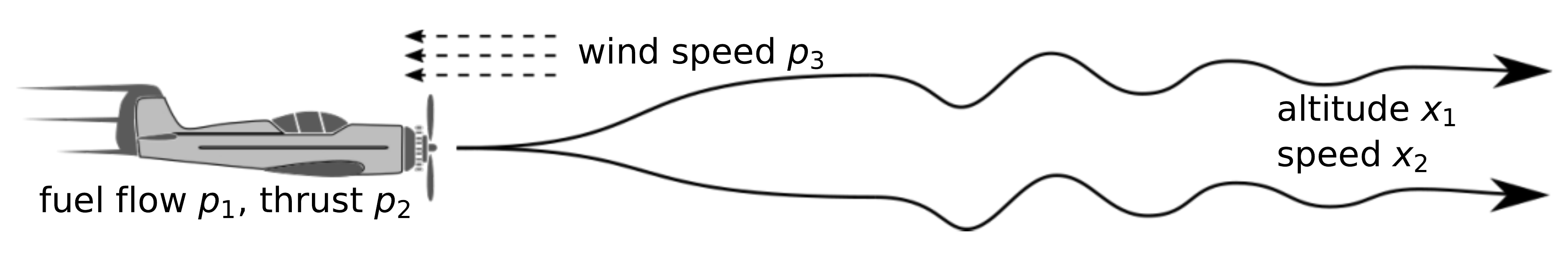}
\caption{\label{fig:airplaneflight}Illustration for our toy model: An airplane flies at an altitude $x_1(t)$, with speed $x_2(t)$. The pilot can influence fuel flow $p_1$ and thrust $p_2$, but not the wind speed $p_3$.}
\end{figure}
    
Constructing minimal realizations of parameter spaces for nonlinear dynamical systems from observations is a long-standing problem~\cite{gutenkunst-2007}. %
Here, we show that jointly smooth functions can be used to identify effective parameters and corresponding steady states, essentially constructing an effective bifurcation diagram for the dynamical system through level set parametrization. 
For an alternative approach to this, see~\cite{holiday-2019}.
%
%
%
Figure~\ref{fig:airplaneflight} illustrates the dynamical system we study in this example. An aircraft in flight is modeled by a nonlinear dynamical system for its altitude $x_1$ and speed $x_2$. The system dynamics is governed by three parameters $p_1,p_2,p_3$ (interpreted as fuel flow, thrust, and wind speed):
\begin{equation}
    \label{eq:sys}
\mathbb{R}^{2}\ni\dot{\boldsymbol{x}}=h_{p_{1},p_{2},p_{3}}(\boldsymbol{x})=J(s(\boldsymbol{x}))\cdot g_{p_{1},p_{2},p_{3}}(s(\boldsymbol{x})).
\end{equation}
The function $h$ describes the behavior of the plane as it oscillates around a target altitude $x_1(t\to\infty)$ and target speed $x_2(t\to\infty)$ given by $s^{-1}(p_1+p_2^3,p_3)$. In this example, we define $h$ through a nonlinear transform $\boldsymbol{y}=s(\boldsymbol{x})$ of a damped, linear pendulum $\dot{\boldsymbol{y}}=g(\boldsymbol{y})$, as 
\begin{equation*}
J(\boldsymbol{y})=\left[\begin{matrix}1 & -2y_{2}\\
-2y_{1}+2y_{2}^{2}  & \quad 1+4y_{1}y_{2}-4y_{2}^{3}
\end{matrix}\right],\quad s(\boldsymbol{x})=\left[\begin{matrix}x_{1}+x_{1}^{4}+2x_{1}^{2}x_{2}+x_{2}^{2}\\
x_{1}^{2}+x_{2}
\end{matrix}\right],
\end{equation*}
and
\begin{equation}
\label{eq:sys_g}    
g_{p_{1},p_{2},p_{3}}(\boldsymbol{y})=\left[\begin{matrix}-2 & 1\\
-1 & -1
\end{matrix}\right]\left[\begin{matrix}y_{1}-\left(p_{1}+p_{2}^{3}\right)\\
y_{2}-p_{3}
\end{matrix}\right].
\end{equation}
Figure \ref{fig:DynamicalSys}(a) presents the phase portraits of the system given $(p_{1},p_{2},p_{3})=(0.1,-0.2,0.2)$ and $(p_{1},p_{2},p_{3})=(0.2,0.3,-0.1)$ at the top and bottom, respectively.
By \eqref{eq:sys_g}, despite having three controlling parameters, the system dynamics depend only on $p_3$ and $p_1 + p_2^3$.

Suppose that the system dynamics, here represented by $h$, are unknown. 
Furthermore, suppose we can only control fuel flow $p_1$ and thrust $p_2$, but not the wind speed $p_3$. For each pair $(p_1,p_2)$, we choose a random $p_3$, simulate the system, and then observe the aircraft's steady-state altitude and speed (in the limit $t\rightarrow\infty$).

The identification of an effective parameter is based on the application of Algorithm \ref{alg:2views} to
the (accessible) parameter space $(p_1,p_2)$ and the collection of observed steady states $(x_1,x_2)$.
For demonstration purposes, we generate $N=2,000$ points $\{ (p_{1},p_{2},p_{3})_{i} \} _{i=1}^{N}$ in the parameter space, uniformly distributed in $[-1,1]^{3}$. For each triplet, we simulate the system \eqref{eq:sys} until convergence to a steady-state point $\left(x_{1},x_{2}\right)_{i}$.
In our analysis, we first construct two kernels: $\boldsymbol{K}_{p}\in\mathbb{R}^{N\times N}$ on the accessible parameters $\left(p_{1},p_{2}\right)_{i}$, and $\boldsymbol{K}_{x}\in\mathbb{R}^{N\times N}$ on the corresponding steady states $\left(x_{1},x_{2}\right)_{i}$. Then, we apply Algorithm \ref{alg:2views} to the two kernels with $d=500$, resulting in the most jointly smooth $M=3$ functions $\boldsymbol{f}_m$.
At the top row of Fig.~\ref{fig:DynamicalSys}(b) we display scatter plots of the top three jointly smooth functions $\boldsymbol{f}_m$ as functions of the unknown combination of parameters $p_1+p_2^3$, where we observe a distinct correspondence. To learn this combination, we plot the scatter of $p_1$ against $p_2$ and color the points according to the obtained jointly smooth functions $\boldsymbol{f}_m$ at the middle row of Fig.~\ref{fig:DynamicalSys}(b). Indeed, we observe that the level sets (marked by red curves) coincide with $p_{1}+p_{2}^{3}=C$ (up to mild boundary effects). Similarly, at the bottom row of Fig.~\ref{fig:DynamicalSys}(b), we depict the scatter plots of the observed steady-states coordinates $x_1$ and $x_2$, which are colored according to the obtained jointly smooth functions $\boldsymbol{f}_m$. We observe that by controlling the combination of parameters $p_1+p_2^3$, one can shift the steady-state along the color gradient, whereas the inaccessible parameter $p_3$ controls the steady-state location along the observed level sets.
In terms of the aircraft illustration, it implies which combinations of altitude and speed can be controlled by the accessible effective parameter.
%

\begin{figure}[h]
    \centering
    \includegraphics[width=1\columnwidth]{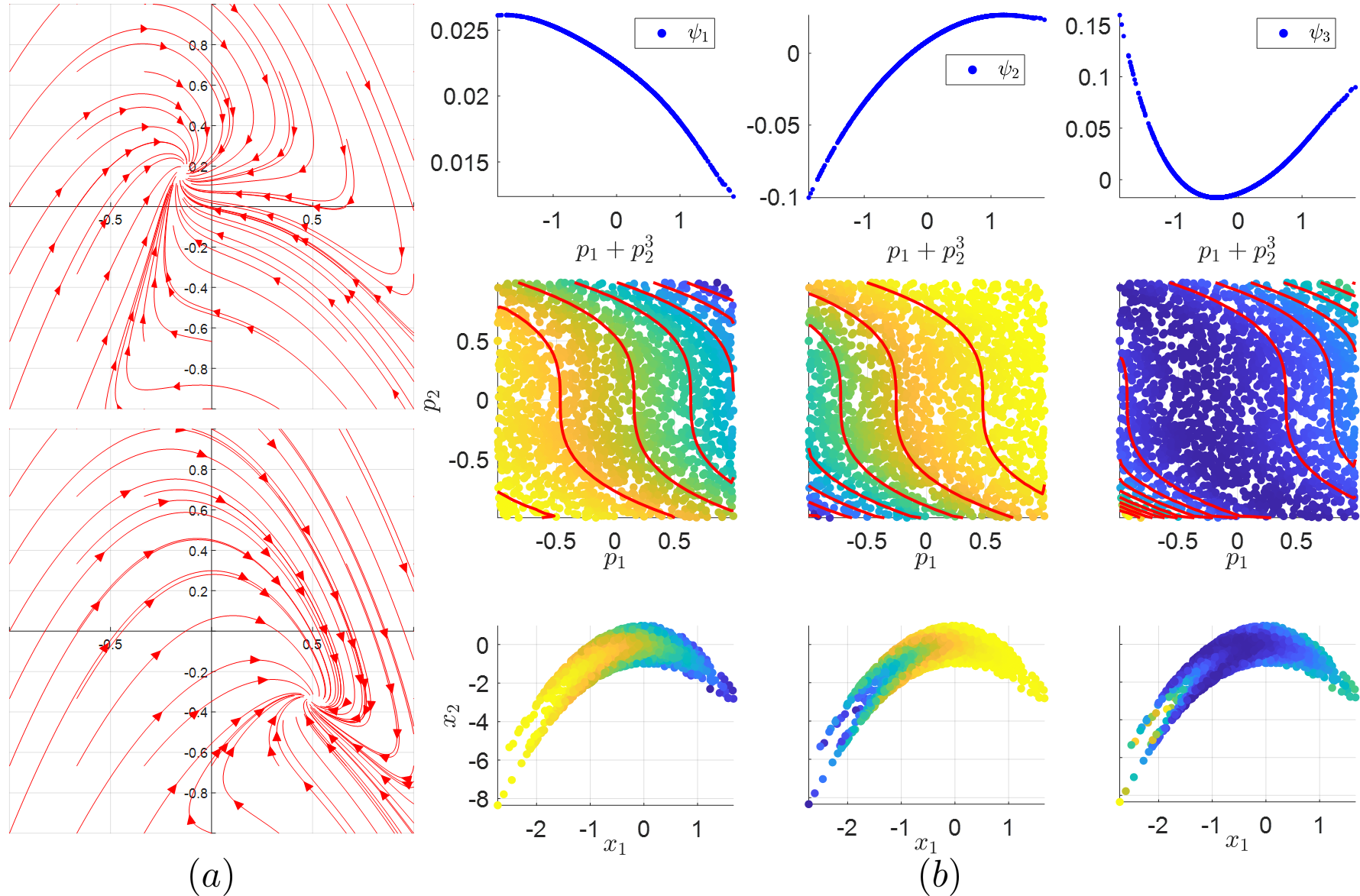}
    \caption{(a) Phase portraits of the dynamical system with $(p_1,p_2,p_3)=(0.1,-0.2,0.2)$ at the top and with $(p_{1},p_{2},p_{3})=(0.2,0.3,-0.1)$ at the bottom. (b) The top row displays the top $3$ jointly smooth functions obtained by Algorithm \ref{alg:2views} as a function of the unknown combination $p_1+p_3^3$. The middle row presents the scatter plots of $p_1$ against $p_2$ colored by the top $3$ jointly smooth functions, where the level sets are marked by red curves. The bottom row depicts the scatter plots of the observed steady-states coordinates $x_1$ and $x_2$, which are colored according to the obtained jointly smooth functions $\boldsymbol{f}_m$.}
    \label{fig:DynamicalSys}
\end{figure}

\section{Conclusions}
In this paper, we presented an approach for multimodal data analysis by introducing a notion of jointly smooth functions on manifolds. We proposed a new spectral algorithm for discovering such jointly smooth functions solely from observations and provided theoretical justification for it. We demonstrated the efficacy of our approach on simulated and real measured data, achieving superior results compared to competing methods.
We believe the generic formulation of the problem and the algorithm facilitates applications in a broad range of fields, reaching well beyond strict data analysis, e.g., in nonlinear dynamical systems analysis, as demonstrated in the paper.


%
%
%
%


\bibliographystyle{siamplain}
\bibliography{Refs}

\appendix
\section{Proof to Lemma 1}
\begin{lemma}
Consider $\boldsymbol{A},\boldsymbol{B}\in\mathbb{R}^{N\times d}$
such that $\boldsymbol{A}^{T}\boldsymbol{A}=\boldsymbol{B}^{T}\boldsymbol{B}=\boldsymbol{I}_{d}$.
Let $\boldsymbol{W}\coloneqq\left[\begin{matrix}\boldsymbol{A} & \boldsymbol{B}\end{matrix}\right]\in\mathbb{R}^{N\times2d}$.
Then, the following decomposition of $\boldsymbol{W}$ 
\[
\boldsymbol{W}=\boldsymbol{U}\boldsymbol{\Sigma}\boldsymbol{V}^{T}
\]
is an SVD, such that
\[
\left\Vert \boldsymbol{A}^{T}\boldsymbol{u}_{i}\right\Vert _{2}^{2}=\left\Vert \boldsymbol{B}^{T}\boldsymbol{u}_{i}\right\Vert _{2}^{2}=\frac{1}{2}\sigma_{i}^{2}\qquad\forall i,
\]
where $\boldsymbol{u}_{i}$ is the $i$th column of $\boldsymbol{U}$
and $\sigma_{i}$ equals $\boldsymbol{\Sigma}\left[i,i\right]$ if
$i\leq2d$ and $0$ otherwise.
In addition, $\boldsymbol{V}=\frac{1}{\sqrt{2}}\left[\begin{matrix}\boldsymbol{Q} & \boldsymbol{Q}\\
\boldsymbol{R} & -\boldsymbol{R}
\end{matrix}\right]\in\mathbb{R}^{2d\times2d}$, $\boldsymbol{\Sigma}^{2}=\left[\begin{matrix}\boldsymbol{I}+\boldsymbol{\Gamma} & \boldsymbol{0}\\
\boldsymbol{0} & \boldsymbol{I}-\boldsymbol{\Gamma}
\end{matrix}\right]\in\mathbb{R}^{2d\times2d}$, and $\boldsymbol{U}=\boldsymbol{W}\boldsymbol{V}\boldsymbol{\Sigma}^{\dagger}\in\mathbb{R}^{N\times2d}$,
where $\boldsymbol{\Sigma}^{\dagger}$ is the pseudo-inverse of $\boldsymbol{\Sigma}$ and $\boldsymbol{A}^{T}\boldsymbol{B}=\boldsymbol{Q}\boldsymbol{\Gamma}\boldsymbol{R}^{T}\in \mathbb{R}^{d\times d}$ is an SVD.
\end{lemma}

\begin{proof}
We have:
\[
\boldsymbol{W}^{T}\boldsymbol{W}=\left[\begin{matrix}\boldsymbol{A}^{T}\\
\boldsymbol{B}^{T}
\end{matrix}\right]\left[\begin{matrix}\boldsymbol{A} & \boldsymbol{B}\end{matrix}\right]=\left[\begin{matrix}\boldsymbol{I} & \boldsymbol{A}^{T}\boldsymbol{B}\\
\boldsymbol{B}^{T}\boldsymbol{A} & \boldsymbol{I}
\end{matrix}\right].
\]
Let $\boldsymbol{v}_{i}=\left[\begin{matrix}\boldsymbol{q}_{i}\\
\boldsymbol{r}{}_{i}
\end{matrix}\right]$ be the $i$th column of $\boldsymbol{V}$, where $\boldsymbol{q}_{i}$ and $\boldsymbol{r}_{i}$ are the $i$th columns of $\boldsymbol{Q}$ and $\boldsymbol{R}$, respectively.
Compute:
\begin{align*}
\boldsymbol{W}^{T}\boldsymbol{W}\boldsymbol{v}_{i} & =\frac{1}{\sqrt{2}}\left[\begin{matrix}\boldsymbol{I} & \boldsymbol{A}^{T}\boldsymbol{B}\\
\boldsymbol{B}^{T}\boldsymbol{A} & \boldsymbol{I}
\end{matrix}\right]\left[\begin{matrix}\boldsymbol{q}_{i}\\
\pm\boldsymbol{r}_{i}
\end{matrix}\right]\\
 & =\frac{1}{\sqrt{2}}\left[\begin{matrix}\boldsymbol{q}_{i}\pm\boldsymbol{A}^{T}\boldsymbol{B}\boldsymbol{r}_{i}\\
\pm\boldsymbol{B}^{T}\boldsymbol{A}\boldsymbol{q}_{i}+\boldsymbol{r}_{i}
\end{matrix}\right]\\
 & =\frac{1}{\sqrt{2}}\left[\begin{matrix}\boldsymbol{q}_{i}\pm\gamma_{i}\boldsymbol{q}_{i}\\
\pm\gamma_{i}\boldsymbol{r}_{i}+\boldsymbol{r}_{i}
\end{matrix}\right]\\
 & =\left(1\pm\gamma_{i}\right)\boldsymbol{v}_{i}\\
 & =\sigma^2_{i}\boldsymbol{v}_{i}.
\end{align*}
In words, $\left\{ \boldsymbol{v}_{i}\right\} $ are the eigenvectors of $\boldsymbol{W}^{T}\boldsymbol{W}$, and thus, they are also the
right singular vectors of $\boldsymbol{W}$.
Since
\[
\boldsymbol{A}=\boldsymbol{W}\left[\begin{matrix}\boldsymbol{I}_{d}\\
\boldsymbol{0}_{d}
\end{matrix}\right],\qquad\boldsymbol{B}=\boldsymbol{W}\left[\begin{matrix}\boldsymbol{0}_{d}\\
\boldsymbol{I}_{d}
\end{matrix}\right],
\]
we have
\begin{align*}
\left\Vert \boldsymbol{A}^{T}\boldsymbol{u}_{i}\right\Vert _{2}^{2} & =\left\Vert \left[\begin{matrix}\boldsymbol{I}_{d} & \boldsymbol{0}_{d}\end{matrix}\right]\boldsymbol{W}^{T}\boldsymbol{u}_{i}\right\Vert _{2}^{2}\\
 & =\left\Vert \left[\begin{matrix}\boldsymbol{I}_{d} & \boldsymbol{0}_{d}\end{matrix}\right]\sigma_{i}\boldsymbol{v}_{i}\right\Vert _{2}^{2}\\
 & =\sigma_{i}^{2}\left\Vert \left[\begin{matrix}\boldsymbol{I}_{d} & \boldsymbol{0}_{d}\end{matrix}\right]\boldsymbol{v}_{i}\right\Vert _{2}^{2}\\
 & =\sigma_{i}^{2}\left\Vert \frac{1}{\sqrt{2}}\boldsymbol{q}_{i}\right\Vert _{2}^{2}\\
 & =\frac{1}{2}\sigma_{i}^{2}.
\end{align*}
Similarly:
\[
\left\Vert \boldsymbol{B}^{T}\boldsymbol{u}_{i}\right\Vert _{2}^{2}=\sigma_{i}^{2}\left\Vert \frac{1}{\sqrt{2}}\boldsymbol{r}_{i}\right\Vert _{2}^{2}=\frac{1}{2}\sigma_{i}^{2},
\]implying that
\[
\Rightarrow\left\Vert \boldsymbol{A}^{T}\boldsymbol{u}_{i}\right\Vert _{2}^{2}=\left\Vert \boldsymbol{B}^{T}\boldsymbol{u}_{i}\right\Vert _{2}^{2}=\frac{1}{2}\sigma_{i}^{2}\qquad\forall i.
\]
\end{proof}

\section{Choosing the Value of $M$: Additional Details}
    \label{SM:threhold}

Let $\boldsymbol{W}_{x},\boldsymbol{W}_{y}\in\mathbb{R}^{N\times d}$ be two matrices with orthogonal columns representing two $d$-dimensional subspaces distributed uniformly over the Grassmann manifold $ \boldsymbol{Gr}(d,N)$. Denote by $0\leq\theta\leq\frac{\pi}{2}$ the smallest principle angle between these two subspaces. Our goal is to approximate the expected value of the cosine of $ \theta $, \textit{i.e.} $ \mathbb{E}[\cos(\theta)] $. It was shown \cite{johnstone2008multivariate} that $ \cos^{2}(\theta) $ is distributed similarly as the largest eigenvalue $ \rho $ of $ (A+B)^{-1}B $, where $ A $ and $ B $ are distributed as Wishart random matrices, with parameters $W_d(\boldsymbol{I}, N-d)$ and $W_d(\boldsymbol{I}, d)$, respectively, and $\boldsymbol{I}$ is the identity. In \cite[p.~2650-2651]{johnstone2008multivariate}, the author computed the limit distribution of $ \rho $ as the dimensions $ N $ and $ d $ grow to infinity. Informally, he showed that 
\begin{equation*}
	\rho = \mu + \sigma Z_1 + O(d^{-4/3}),
\end{equation*}
where $\mu = \sin^2(\varphi)$ and $\sigma \in O(N^{-2/3})$,
such that $\sin^2(\varphi/2) = \frac{d-1/2}{N-1}$
and $ Z_1 $ is a Tracy-Widom random variable. We observe that when the number of samples $ N $ is large (and hence also $ d $), the distribution of $ \rho $ exhibits a concentration of measure effect, in which only values at the close neighborhood around $ \mu $ are probable. Therefore, $ \mu $ is a good estimation for $ \rho $ in large datasets. Note that 
\begin{align*}
	\mu = \sin^2(\varphi) & = 4\sin^2(\varphi/2)\cos^2(\varphi/2)\\
	& = 4\sin^2(\varphi/2) \left(1-\sin^2(\varphi/2) \right)\\
	& = 4\frac{d-1/2}{N-1} \left(1-\frac{d-1/2}{N-1} \right)\\
	& = 4\frac{d-1/2}{N-1} \frac{N-d-1/2}{N-1}.
\end{align*}
Thus,
\begin{align*}
\mathbb{E}[\rho] \approx 4\frac{(d-1/2)(N-d-1/2)}{(N-1)^2}.
\end{align*}
Now, according to Jensen's inequality
\begin{align*}
\mathbb{E}[ \cos(\theta) ] = \mathbb{E}[\sqrt{\cos^2(\theta)} ] \leq \sqrt{\mathbb{E}[ \cos^2(\theta) ]} = \sqrt{\mathbb{E}[\rho]} \approx 2\frac{ \sqrt{d-1/2} \sqrt{N-d-1/2}}{N-1}.
\end{align*}

\section{Sleep Stage Identification: Additional Details}
Here, we describe the procedure applied to the raw measurements, giving rise to the six kernels used as inputs to Algorithm 2 in the paper. Note that the same procedure is applied to all subjects separately.

Let $\boldsymbol{s}_{i}^{\left(k\right)}$ denote the signal from the $i$th sensor of the $k$th subject recorded in a single night.
First, we remove most of the awake stage at the beginning and the end according to the expert annotations, remaining with approximately $5$ hours of recordings.
In addition, we remove unlabelled segments.
Next, to each signal from each sensor we apply a simple DC notch filter and compute its spectrogram denoted as $\boldsymbol{S}_{i}^{\left(k\right)}\in\mathbb{R}^{1,025\times{T}_{k}}$, where $1025$ is the number of frequency bins and $T_k$ is the number of time frames that depends on the varying lengths of the recordings (and is typically about $3000$). We use Hamming analysis windows of length $5$ seconds with $1$ second overlap.
The columns of the spectrogram (the frequency profile in each time frame) are projected into a $10$ dimensional space using PCA, resulting in the features $\boldsymbol{X}_{i}^{\left(k\right)}\in\mathbb{R}^{10\times{T}_{k}}$.
Finally, we construct the kernel $$\boldsymbol{K}_{i}^{\left(k\right)}\left[a,b\right]=\exp\left(-\frac{d^{2}\left(\boldsymbol{x}_{a},\boldsymbol{x}_{b}\right)}{2\sigma_{i,k}^{2}}\right),\qquad\boldsymbol{K}_{i}^{\left(k\right)}\in\mathbb{R}^{{T}_{k}\times{T}_{k}},$$
where $\boldsymbol{x}_{a}$ is the $a$th column of $\boldsymbol{X}_{i}^{\left(k\right)}$, $d\left(\cdot,\cdot\right)$ is the non-linear Mahalanobis distance proposed in \cite{singer2008non} and we set $\sigma_{i,k}=\frac{1}{2}\text{median}\left\{ d\left(\boldsymbol{x}_{a},\boldsymbol{x}_{b}\right)\right\} _{a,b}$.

The four subjects used in our experiment are n1, n2, n3 and n5. We omit n4 since its recording is short and the duration of the sleep stages is unbalanced.

\end{document}

%% file: ex_shared.tex
\usepackage[utf8]{inputenc}

\usepackage{lipsum}
\usepackage{epstopdf}
\ifpdf
  \DeclareGraphicsExtensions{.eps,.pdf,.png,.jpg}
\else
  \DeclareGraphicsExtensions{.eps}
\fi

\usepackage{graphicx}
\graphicspath{{./Figs/}{./}}

\usepackage{amsfonts}       
\usepackage{bm}
\usepackage{amsmath}
\usepackage{nicefrac}       
\usepackage{ulem}
\usepackage{hyperref}
\usepackage{mathtools}
\usepackage{color, colortbl}
\definecolor{TitleGray}{gray}{.8}
\definecolor{RowGray}{gray}{.9}

\newtheorem{remark}{Remark}

\newcommand{\thetitle}{Spectral Discovery of Jointly Smooth Features for Multimodal Data}

\headers{\thetitle}{F. Dietrich, O. Yair, R. Mulayoff, R. Talmon, and I. G. Kevrekidis}

\hypersetup{
  pdftitle={\thetitle},
  pdfauthor={Felix Dietrich, Or Yair, Rotem Mulayoff, Ronen Talmon, Ioannis G. Kevrekidis}
}

\title{\thetitle\thanks{Submitted to the editors on \today. Authors F.D. and O.Y. contributed equally to this work.\funding{The work of O.Y. and R.T. was funded by the European Union’s Horizon 2020 research and innovation programme under grant agreement No. 802735-ERC-DIFFOP. F.D. and I.G.K. were funded by the U.S. Army Research Office under contract/grant number W911NF1710306 (an ARO MURI) and the DARPA ATLAS program. The work of IGK was partially supported by an ARO MURI (Dr. M. Munson) and the DARPA ATLAS program (Dr. J. Zhou).}}}

\author{Felix Dietrich\thanks{Department of Informatics, Technical University of Munich
  (\email{felix.dietrich@tum.de}).}
 \and
Or Yair\thanks{Viterbi Faculty of Electrical Engineering, Technion, Israel Institute of Technology (\email{oryair@campus.technion.ac.il, ronen@ee.technion.ac.il})}
 \and
Rotem Mulayoff\footnotemark[3]
 \and
Ronen Talmon\footnotemark[3]
\and Ioannis G.
Kevrekidis\thanks{Department of Chemical and Biomolecular Engineering and Department of Applied Mathematics and Statistics, Johns Hopkins University and JHMI 
  (\email{yannis@princeton.edu}).}}